\RequirePackage{fix-cm}
\documentclass[twocolumn]{svjour3}          % twocolumn
\smartqed  % flush right qed marks, e.g. at end of proof
\usepackage{algorithm}
\usepackage{algorithmic}
\usepackage{float}
\usepackage[utf8]{inputenc} % allow utf-8 input
\usepackage[T1]{fontenc}    % use 8-bit T1 fonts
\usepackage{hyperref}       % hyperlinks
\usepackage{url}            % simple URL typesetting
\usepackage{booktabs}       % professional-quality tables
\usepackage{amsmath}        % mathematical equations
\usepackage{amsfonts}       % blackboard math symbols
\usepackage{nicefrac}       % compact symbols for 1/2, etc.
\usepackage{microtype}      % microtypography
\usepackage{xcolor}         % colors
\usepackage{wrapfig}
\usepackage{multirow}
\usepackage[counterclockwise]{rotating}
\usepackage{graphicx}
\usepackage{tabularx}
\usepackage{array}
\usepackage{subfigure}
\usepackage[numbers]{natbib}

\begin{document}

\title{UniWind: Day-Ahead Wind Power Forecasting via Physics-Informed State Routing%\thanks{Grants or other notes
}
% \subtitle{Do you have a subtitle?\\ If so, write it here}


\author{Ronghui Xu \and
  Tongxin Wu\and
  Guozhen Zhang\and
  Yihan Li\and
  Chenjuan Guo\and
  Bin Yang\and
  Yong Li
}


\institute{Ronghui Xu \and 
            Chenjuan Guo \and
            Bin Yang
                \at
              East China Normal University, Shanghai, China
           \and
           Tongxin Wu \and  Yihan Li \at
              Southern University of Science and Technology, Shenzhen, China
              \and
              Guozhen Zhang \at
              TsingRoc.ai, Beijing, China
              \and
              Yong Li \at
              Tsinghua University, Beijing, China
}

\date{Received: date / Accepted: date}

\maketitle

\begin{abstract}
Accurate day-ahead wind power forecasting is essential for cost-effective power-system operation. Wind power generation is primarily driven by future meteorological conditions while retaining temporal dependencies in power generation. In practice, observed wind-farm power reflects both meteorology-driven generation potential and deviations associated with local environmental and operational factors. Existing physical models provide useful constraints but adapt poorly across wind farms, whereas data-driven models can capture rich correlations but do not explicitly structure these two sources of variation.
To bridge this gap, we propose UniWind, a wind power forecasting framework that integrates physics-informed generation priors with data-driven state-aware corrections.
UniWind first employs a Physical Prior Estimator to construct a site-calibrated physical prior by combining a site-conditioned monotone warp with a shared physical power curve. It further applies a physical upper-bound constraint to shape this prior as a soft envelope of available wind power generation. UniWind then uses a Latent State Encoder to construct latent power-state representations from historical prior-power discrepancies and future meteorological context, and transforms the physical prior into final power forecasts through a State-aware Power Corrector with pseudo-label-guided routing and bounded expert correction. Experiments on 24 real-world datasets demonstrate UniWind's strong forecasting accuracy, robustness, and generalization across full-shot, cross-farm zero-shot, and cross-region zero-shot settings.
\end{abstract}

\section{Introduction}
\label{sec:intro}

Day-ahead wind power forecasting plays a central role in cost-effective power-system operation, informing key decisions such as generation scheduling, reserve procurement, and electricity-market bidding~\cite{wu2026electric}. To support these decisions, forecasts must be generated ahead of power delivery within predefined issuance deadlines~\cite{hanifi2020critical,xu2025cross}. The forecasting task involves predicting future wind-farm power generation from historical observations, site information, and numerical weather prediction (NWP) forecasts available at the time of issuance. Compared with conventional time series forecasting~\cite{stitsyuk2025xpatch,liuitransformer,wangtimemixer,nietime,fang2026efficient} based primarily on historical sequence extrapolation, this task places greater emphasis on meteorological conditions over the future delivery interval. For example, high power output on previous days does not imply high future generation when the forecast wind speed is low. Historical power nevertheless remains informative about recent operating patterns. 
Therefore, day-ahead wind power forecasting is better characterized as predicting realized power output conditioned on future meteorological forecasts, while accounting for operating-state information embedded in historical observations.

\begin{figure}[!htbp]
    \centering
    \includegraphics[width=0.98\linewidth]{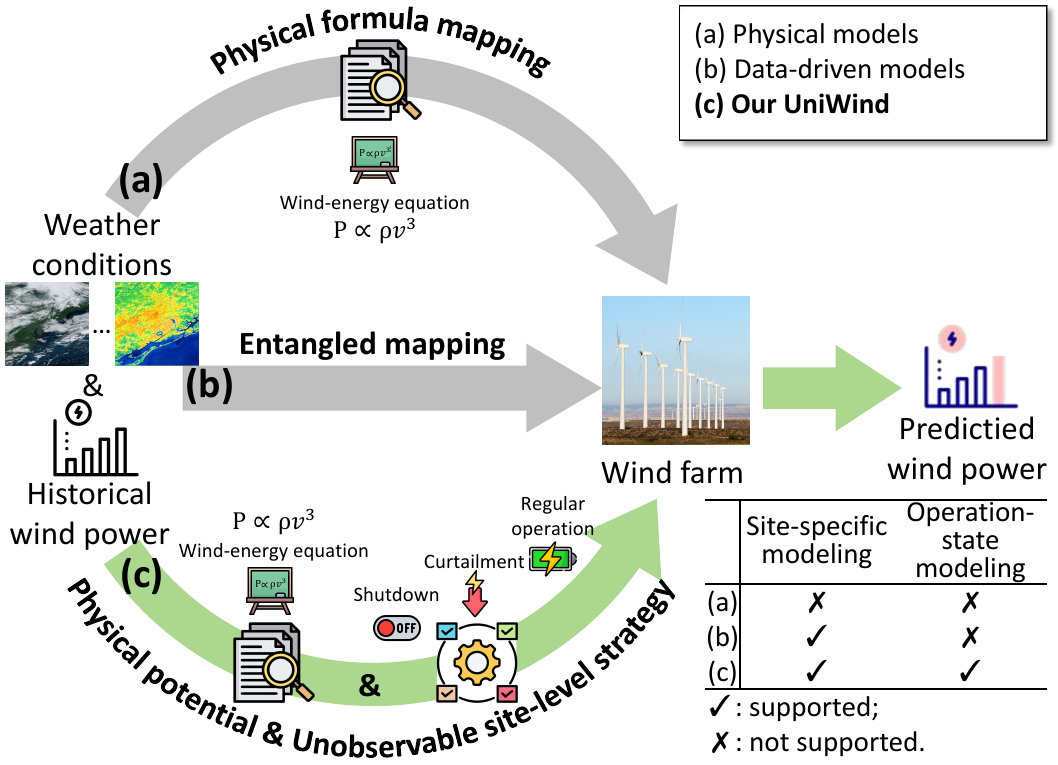}
    \caption{Comparison of UniWind with other wind power forecasting models.}
    \label{fig:intro_concept}
\end{figure}

\begin{figure}[!htbp]
    \centering
    \includegraphics[width=0.65\linewidth]{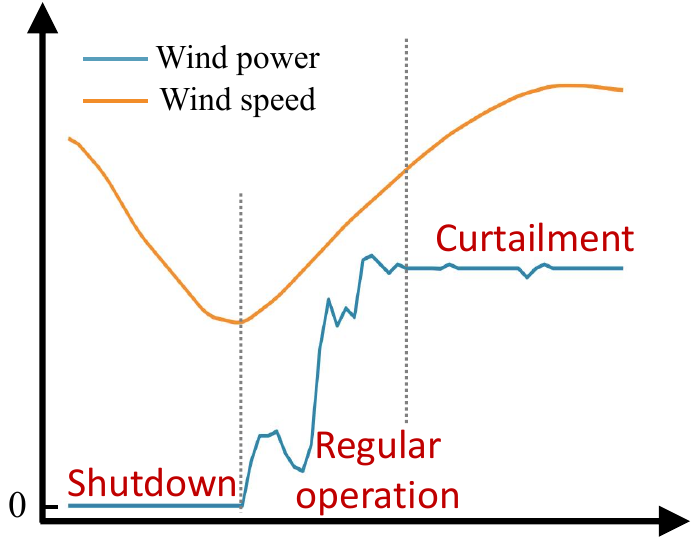}
    \caption{Operational states in a wind power sequence.}
    \label{fig:state_case}
\end{figure}

Existing wind power forecasting methods can be categorized into physically driven and data-driven approaches. Physical models~\cite{focken2001previento,li2013physical,wang2018short}, detailed in Figure~\ref{fig:intro_concept}(a), provide reliable power predictions through physical constraints and structured inductive biases. However, fixed formulations may not adequately capture site-specific responses and time-varying deviations from the nominal wind-to-power relationship. In particular, as shown in Figure~\ref{fig:state_case}, the power supported by meteorological conditions can differ substantially from the power actually delivered. Turbines may operate in shutdown states with zero output, under regular conditions in which power follows the wind-speed response, or under curtailment that limits output below available power. These latent-state changes alter the wind-to-power mapping, limiting the adaptability of fixed physical formulations across different wind farms and operating conditions. Data-driven methods~\cite{sideratos2007advanced,ozkan2015novel,karijadi2023wind,zhang20252dxformer,chang2025multivariate} learn flexible relationships among weather, historical power, and future generation from observed data.  Nevertheless, as shown in Figure~\ref{fig:intro_concept}(b), they typically represent meteorology-driven generation potential and hidden-state effects within a single unconstrained forecasting mapping.
Consequently, low observed power may reflect weak wind availability, an operating constraint, or their combination.  Without a reference for available power, these sources of variation can be difficult to distinguish.  This ambiguity becomes particularly relevant when site characteristics and operating patterns differ across wind farms.  These considerations motivate explicitly  modeling the dependency from meteorological conditions to physical power potential and further to realized power under hidden-state interventions, as illustrated in Figure~\ref{fig:intro_concept}(c). However, constructing such a model poses two key challenges:

\textit{\textbf{First, estimating power generation potential requires balancing universal physical principles with site-specific adaptation.}} 
Wind power generation follows physical relationships that provide reusable structure across wind farms. However, real-world generation potential is also shaped by site-specific factors, such as local environmental effects, that are often unobserved.
Existing fixed physical formulas~\cite{prosper2019wind,wang2018short} may be too rigid to capture site-dependent wind-to-power responses. Unconstrained data-driven methods~\cite{buhan2015multistage,hong2019hybrid} offer greater flexibility but may absorb noisy correlations or violate the intended physical structure. A further difficulty is that available power is not directly observed: recorded generation may already be reduced by shutdown or curtailment.  Fitting these observations without appropriate constraints can therefore mix generation potential with operating effects. 
Therefore, an effective physical prior should preserve wind-energy physics while enabling site-level calibration under appropriate constraints.

\textit{\textbf{Second, predicting realized power requires modeling the effects of unobservable, time-varying future operational states.}} 
Similar available-power levels can correspond to substantially different realized outputs under regular operation, shutdown, or curtailment.  A low-power observation alone therefore does not establish whether reduced wind availability or an operating condition is responsible. 
Existing hidden-state-aware methods~\cite{liuimproving,zhou2023deep} typically infer latent representations directly from historical sequences, which can leave these factors entangled.
In day-ahead forecasting, the difficulty extends beyond interpreting past observations because future operating conditions are unavailable when the forecast is issued.   Moreover, a single correction mechanism may average over distinct power-response patterns and obscure abrupt changes.  
This motivates inferring future states and applying state-specific corrections to the physical prior.

To this end, we propose \textbf{UniWind}, a physics-informed state-routing model for day-ahead wind power forecasting. UniWind constructs a physically motivated power envelope and learns bounded corrections conditioned on weather and historical power deviations to forecast realized generation. 

\textit{To address the first challenge}, UniWind introduces a Physical Prior Estimator that combines a site-conditioned monotone warp and a shared physical power curve to construct a site-calibrated physical prior. The warp calibrates the effective wind-speed response while preserving wind-speed ordering, whereas the shared curve retains reusable wind-generation structure.  Physical upper-bound regularization encourages the resulting prior to cover observed power while remaining close to it during rule-defined regular operation.  The prior thus serves as a site-calibrated soft upper envelope that guides state inference and power correction.

\textit{To address the second challenge}, UniWind uses a Latent State Encoder to encode historical operating patterns from discrepancies between observed power and the physical prior. It then uses future NWP context to retrieve historical state patterns associated with similar meteorological conditions. The resulting representations guide a State-aware Power Corrector, which combines bounded expert outputs using the routing probabilities to obtain realized power forecasts.  Different experts accommodate shutdown-like, curtailment-like, and regular power-response patterns, allowing the correction to vary with the inferred latent power state. 
Our contributions are summarized as follows:
\begin{itemize}
    \item We propose UniWind, a day-ahead wind power forecasting model that combines available-power estimation with state-dependent correction to predict realized generation.
    \item We develop a Physical Prior Estimator that combines a shared physical power curve with site-conditioned monotone calibration. Soft upper-envelope regularization shapes the resulting prior into a reference for available generation.
    \item We introduce a Latent State Encoder and a State-aware Power Corrector to model latent power-states and their effects on wind-farm generation. Weather-conditioned state retrieval and weakly supervised routing guide bounded expert corrections without requiring state labels at inference time.
    \item We conduct extensive experiments on 24 real-world wind-farm datasets under full-shot, cross-farm zero-shot, and cross-region zero-shot settings, and demonstrate the accuracy and robustness of UniWind.
\end{itemize}

The remainder of this paper is organized as follows. Section~2 introduces the day-ahead wind power forecasting task and formalizes the problem. Section~3 presents the UniWind framework, including its physical prior estimation, latent-state encoding, state-aware power correction, and training objectives. Section~4 describes the experimental setup and reports the empirical results and analyses. Section~5 reviews related work, and Section~6 concludes the paper.

\section{Preliminaries}
\label{sec:preliminaries}

\subsection{Day-Ahead Wind Power Forecasting Task}
\label{sec:day_ahead_task}

In this study, day-ahead wind power forecasting is defined as the prediction of a wind farm's realized power output over a predefined delivery interval on the following day, using exclusively information available at the forecast issuance time. Accordingly, future meteorological observations and other variables realized after the issuance time are excluded from the admissible input set. This formulation distinguishes day-ahead wind power forecasting from generic long-horizon time-series extrapolation and meteorological forecasting, while explicitly specifying the information constraints and prediction target considered in this study.

Let $\tau_0$ denote the forecast issuance time, and let
$p_t$ denote the measured power output during the
interval beginning at time index $t$ (in MW), and $P_{\mathrm{cap}}$ its rated capacity.
We define $T_1$ as the number of historical time steps used
as input and $T_2$ as the forecast horizon, measured in
discrete time steps from $\tau_0$ to the final timestamp
of the target delivery interval.

The historical power input consists of the most recent
$T_1$ observations before $\tau_0$:
\begin{equation}
    \boldsymbol{x}
    =
    [p_{\tau_0-T_1},\ldots,p_{\tau_0-1}],
    \qquad
    \boldsymbol{x}\in\mathbb{R}^{T_1}.
\end{equation}

The static site features $\boldsymbol{c}\in\mathbb{R}^{N_s}$
include the wind farm's longitude, latitude, and rated
capacity. The numerical weather prediction (NWP) input covers both the historical context window and the future forecast horizon:
\begin{equation}
\begin{aligned}
    \boldsymbol{w}
    &=
    [\boldsymbol{w}_{\tau_0-T_1},\ldots,
    \boldsymbol{w}_{\tau_0+T_2-1}],\\
    \boldsymbol{w}
    &\in\mathbb{R}^{(T_1+T_2)\times N_w},
\end{aligned}
\end{equation}
where $N_w$ denotes the number of meteorological variables. The historical segment provides meteorological context, while the future segment contains forecasts available at the issuance time $\tau_0$. We consistently use NWP data for both segments to mitigate the distribution shift that could arise from combining historical meteorological observations with future weather forecasts.

Given the NWP sequence $\boldsymbol{w}\in\mathbb{R}^{(T_1+T_2) \times N_w}$, historical power sequence $\boldsymbol{x}\in\mathbb{R}^{T_1}$, and site features $\boldsymbol{c}\in\mathbb{R}^{N_s}$, we learn a forecasting model $f_\theta$ to predict wind power $\hat{\boldsymbol{y}}\in\mathbb{R}^{T_2}$:
\begin{equation}
    \hat{\boldsymbol{y}}
    =
    f_\theta\left(
    \boldsymbol{w},
    \boldsymbol{x},
    \boldsymbol{c}
    \right),
\end{equation}
where $T_1$ and $T_2$ are the lengths of the historical and forecasting horizons, respectively. $N_w$ is the number of meteorological features, and $N_s$ is the number of static site features, including longitude, latitude, and rated power. 
The corresponding realized future power sequence is
$\boldsymbol{y}
    =
    [p_{\tau_0},p_{\tau_0+1},\ldots,
    p_{\tau_0+T_2-1}]^\top
    \in \mathbb{R}^{T_2}.
$

\subsubsection{Fixed Lead Time and Delivery Interval}

The day-ahead forecasting horizon consists of a fixed
lead-time gap followed by the target delivery interval.
Let $T_{\mathrm{gap}}$ denote the number of time steps
between the forecast issuance and the start of the
delivery interval, and $T_{\mathrm{day}}$ denote the
number of time steps within the target delivery day.
The total forecast horizon is therefore defined as
\begin{equation}
    T_2 = T_{\mathrm{gap}} + T_{\mathrm{day}}.
\end{equation}

Given the forecast issuance time $\tau_0$, the target
delivery interval is specified by
\begin{equation}
    \mathcal{T}_{\mathrm{DA}}
    =[\tau_0+T_{\mathrm{gap}},\ldots,\tau_0+T_2-1].
\end{equation}

Accordingly, the day-ahead target and its corresponding
prediction are defined as
\begin{equation}
\begin{aligned}
    \boldsymbol{y}^{\mathrm{DA}}
    &= \boldsymbol{y}_{T_{\mathrm{gap}}+1:T_2},\\
    \hat{\boldsymbol{y}}^{\mathrm{DA}}
    &= \hat{\boldsymbol{y}}_{T_{\mathrm{gap}}+:T_2},\\
    \boldsymbol{y}^{\mathrm{DA}},
    \hat{\boldsymbol{y}}^{\mathrm{DA}}
    &\in \mathbb{R}^{T_{\mathrm{day}}},
\end{aligned}
\end{equation}
where $\boldsymbol{y}_j=\boldsymbol{p}_{\tau_0+j-1}$ for $j=1,\ldots,T_2$.
The initial $T_{\mathrm{gap}}$ future steps are
excluded from the delivery-day prediction target.
Thus, only the power output over the predefined
delivery interval is considered for day-ahead
forecast evaluation.

\subsubsection{Information Set and Forecasting Protocol}

For each forecasting instance, the admissible information is restricted to data available at the forecast issuance time $\tau_0$. The information set is defined as
\begin{equation}
    \mathcal{I}_{\tau_0}
    =(\boldsymbol{x},\boldsymbol{w},\boldsymbol{c}),
\end{equation}
where $\boldsymbol{x}$ contains historical power measurements before $\tau_0$, $\boldsymbol{w}$ comprises historical and future numerical weather prediction (NWP) covariates, and $\boldsymbol{c}$ represents static wind-farm attributes. All NWP covariates must be obtained from forecast cycles available no later than $\tau_0$. Future meteorological observations, analyses, and reanalysis fields unavailable at issuance are excluded.

The forecasting protocol imposes three constraints.
\begin{itemize}
    \item Each forecast is generated at a predefined issuance time relative to the target delivery interval and remains unchanged thereafter. 
    \item No realized information associated with intervals beginning at or after $\tau_0$, including power measurements, operational states, or meteorological observations, is admissible as forecasting input. 
    \item The lead-time gap $T_{\mathrm{gap}}$ and delivery interval length $T_{\mathrm{day}}$ are fixed across forecasting instances, with predictions evaluated exclusively over $\mathcal{T}_{\mathrm{DA}}$.
\end{itemize}

Under these constraints, day-ahead wind power forecasting aims to predict the realized power output over the predefined delivery interval using only $\mathcal{I}_{\tau_0}$, without incorporating subsequently observed information.

\subsection{Rule-derived State Pseudo-labels}
\subsubsection{State Interpretation}
\label{sec:state_validation}

To facilitate application across heterogeneous wind farms, we derive heuristic state pseudo-labels based on observed wind-power behavior in the training set. This design avoids reliance on detailed operational records, such as supervisory control and data acquisition (SCADA) status codes, dispatch instructions, maintenance logs, and fault records, which are often difficult to obtain across different farms and operators. Moreover, such records may not fully capture all operating conditions even when available. The resulting pseudo-labels serve as weak supervision for model training rather than ground-truth annotations of actual operating states. As detailed operational records are unavailable in the datasets considered, these pseudo-labels cannot be directly validated against recorded operational events. Accordingly,  ``shutdown,'' ``curtailment,'' and ``regular'' denote rule-defined categories rather than verified operational states.

\subsubsection{Construction of Rule-Derived State Pseudo-Labels}
\label{sec:state_prior_labels}
To obtain training-time supervision without manually annotated operating states, we derive pseudo-labels from wind-power behavior. For each wind farm, we first estimate an empirical available-power curve from wind-speed--power pairs in its training data. Let \(\bar v_t\) and $p_t$ denote the representative wind speed and observed power at timestamp $t$, respectively. We divide the wind-speed range between the 1st and 99.5th percentiles into bins, then use the $0.9$ quantile of clipped power in each bin as an estimate of available power:
\begin{equation}
    q_j
    =
    \operatorname{Quantile}_{0.9}
    \left(
    \left\{
    p_t^{\mathrm{clip}}:\bar v_t\in B_j
    \right\}
    \right),
\end{equation}
where \(B_j\) is the \(j\)-th wind-speed bin and \(q_j\) is its empirical available-power estimate. Missing bins are linearly interpolated, the resulting curve is smoothed by short moving-window filters, and monotonicity is enforced by a cumulative maximum. Interpolation at \(\bar v_t\) then gives the expected available power \(q_t^{\mathrm{exp}}\in[0,P_{\mathrm{cap}}]\).

We compare measured power with the expected available power to construct the state rules. Specifically, we define the realized-to-available power ratio and local power variability as:
\begin{equation}
\begin{aligned}
    r_t^{\mathrm{exp}}=
    \frac{p_t^{\mathrm{clip}}}{\max(q_t^{\mathrm{exp}},10^{-6})},\quad
    \sigma_t^{(4)}=
    \operatorname{Std}
    \left(
    p_{t-1:t+2}^{\mathrm{clip}}
    \right).
\end{aligned}
\end{equation}
The ratio \(r_t^{\mathrm{exp}}\) measures realized output relative to expected available power, while \(\sigma_t^{(4)}\) is a four-step local standard deviation that helps identify sustained, nearly flat segments. We use the following thresholds:
\begin{equation}
\begin{aligned}
    \eta_{\mathrm{avail}}&=0.2,
    \quad
    \eta_{\mathrm{pow}}=0.05, \\
    \quad
    \eta_{\mathrm{ratio}}&=0.7,
    \quad
    \eta_{\mathrm{std}}=0.01.
\end{aligned}
\end{equation}
A timestamp is considered to have sufficient available power when \(q_t^{\mathrm{exp}}\ge \eta_{\mathrm{avail}}P_{\mathrm{cap}}\). Given this condition, the shutdown and curtailment masks are defined as
\begin{equation}
    m_t^{\mathrm{sd}}
    =
    \left[
    q_t^{\mathrm{exp}}\ge \eta_{\mathrm{avail}}P_{\mathrm{cap}}
    \right]
    \wedge
    \left[
    p_t^{\mathrm{clip}}\le \eta_{\mathrm{pow}}P_{\mathrm{cap}}
    \right],
\end{equation}
\begin{equation}
\begin{aligned}
    m_t^{\mathrm{cur}}
    &=
    \left[
    q_t^{\mathrm{exp}}\ge \eta_{\mathrm{avail}}P_{\mathrm{cap}}
    \right]
    \wedge
    \left[
    r_t^{\mathrm{exp}}\le \eta_{\mathrm{ratio}}
    \right]\\
    &\wedge
    \left[
    \sigma_t^{(4)}\le \eta_{\mathrm{std}}P_{\mathrm{cap}}
    \right].
\end{aligned}
\end{equation}
The mask \(m_t^{\mathrm{sd}}\) identifies near-zero output despite sufficient available power, whereas \(m_t^{\mathrm{cur}}\) identifies sustained under-production with low local variability. We assign the final pseudo-label \(l_t\in\{0,1,2\}\) as follows:
\begin{equation}
    l_t
    =
    \begin{cases}
    0, & m_t^{\mathrm{sd}},\\
    1, & m_t^{\mathrm{cur}}\ \mathrm{and}\ \neg m_t^{\mathrm{sd}},\\
    2, & \mathrm{otherwise},
    \end{cases}
\end{equation}
where \(0\), \(1\), and \(2\) denote the shutdown-like, curtailment-like, and regular-generation categories, respectively. If the shutdown and curtailment masks overlap, the shutdown-like label takes priority.

\section{Method}
\subsection{Overall Framework}
\label{sec:latent_state_pipeline}
\begin{figure*}[tp]
	\centering
 \includegraphics[width=0.99\linewidth]{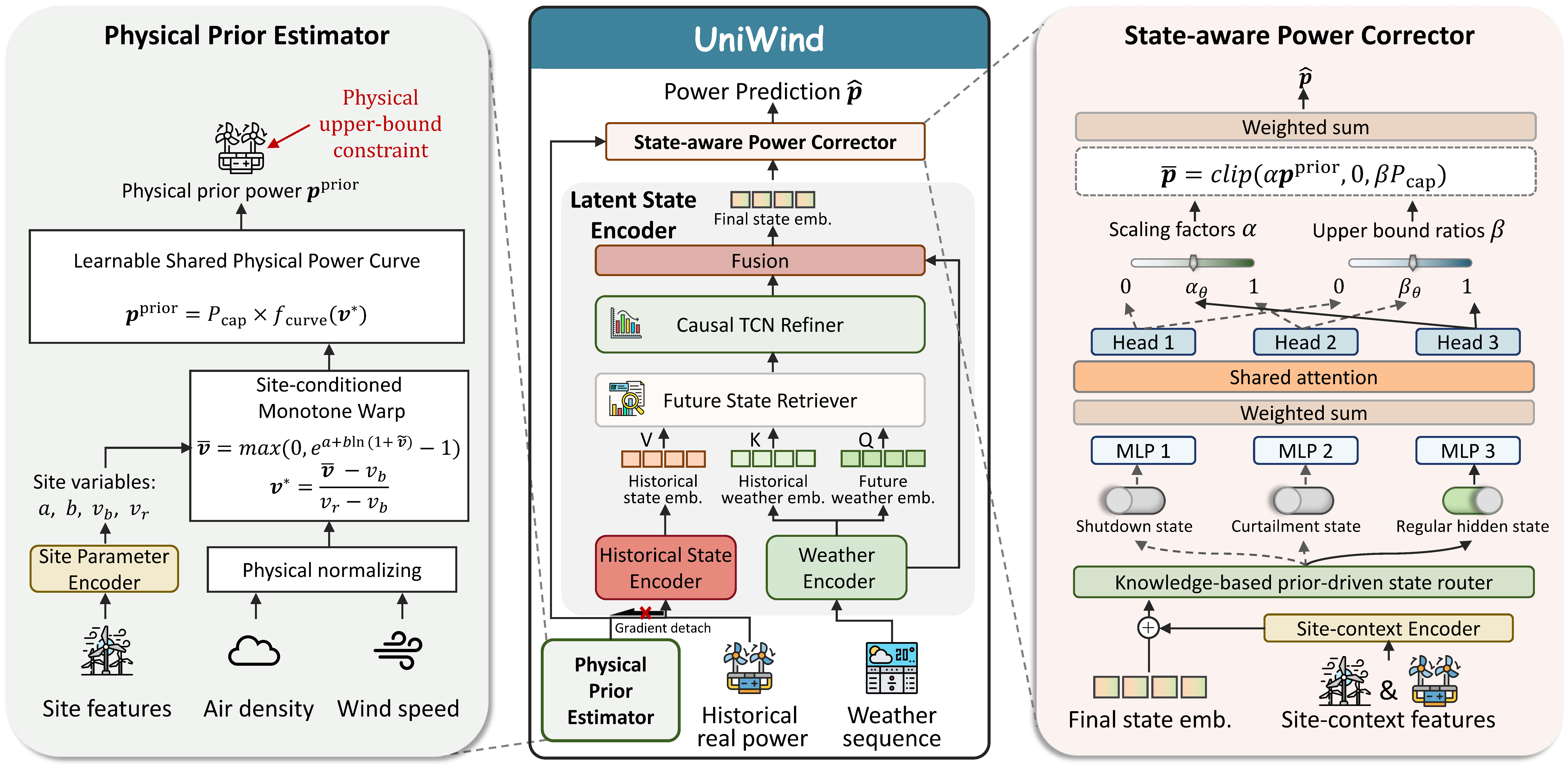}
  \caption{Overall framework of UniWind.}
 \label{fig:framework}
\end{figure*}

As shown in Figure~\ref{fig:framework}, UniWind predicts realized power through a physical-prior pathway followed by state-aware correction. The Physical Prior Estimator normalizes wind speed for air density, calibrates it to the target site, and applies a shared learnable power curve to estimate meteorology-driven available power; an upper-bound constraint encourages this estimate to act as a soft envelope. The Latent State Encoder represents historical operating patterns from discrepancies between observed power and the prior, retrieves future state representations under similar meteorological conditions, and fuses them with weather features. The State-aware Power Corrector then routes each time step to pseudo-label-supervised experts and aggregates bounded, state-specific corrections to obtain the final forecast. During training, rule-derived pseudo-labels provide weak supervision for the state router and identify regular-operation samples used by the prior constraint; they are not required as inference-time inputs.

\subsection{Physical Prior Estimator}
\label{sec:physical_prior_estimator}
Wind power follows shared physical principles but varies across sites because of local terrain, turbine characteristics, and environmental conditions. The Physical Prior Estimator therefore constructs a site-calibrated, physically motivated upper-envelope prior. This prior provides the reference power sequence for the Latent State Encoder and the baseline that the State-aware Power Corrector subsequently adjusts.

\textbf{Physical normalization.}
To account for variations in air density, we first normalize wind speed to a reference air density using the density-aware adjustment~\cite{pandit2020gaussian}:
\begin{equation}
    \tilde{\boldsymbol{v}}
    =
    \boldsymbol{v}
    \odot
    \left(\frac{\boldsymbol{\rho}}{\rho_0}\right)^{1/3},
\end{equation}
where $\boldsymbol{v}, \boldsymbol{\rho} \in \mathbb{R}^{T_1+T_2}$ denote the wind-speed and air-density sequences, respectively, both spanning the historical and future horizons. $\rho_0 = 1.225~\mathrm{kg/m^3}$ is the reference air density, and $\odot$ denotes element-wise multiplication. This adjustment follows the wind-energy relationship $P_{\mathrm{wind}} \propto \rho v^3$ and mitigates first-order variations in wind-power potential caused by changes in air density.

\textbf{Site-conditioned monotone warp.} The density-adjusted speed does not by itself account for site-specific wind-to-power differences. We therefore use a site-conditioned monotone warp to map equivalent wind speed to effective wind speed while preserving its ordering:
\begin{equation}
\begin{aligned}
    \hat{\boldsymbol{v}}
    &=
    \max\left(0,\; \exp(a)(1+\tilde{\boldsymbol{v}})^{\exp(b)} - 1\right),\\
    \qquad
    [a,b]^\top &= h_{\phi}(\boldsymbol{c}) \in \mathbb{R}^{2}.
\end{aligned}
\end{equation}
Here, \(\boldsymbol{c}\in\mathbb{R}^{N_s}\) denotes the site features, \(h_{\phi}(\cdot)\) is an MLP-based site-conditioned parameter head, \(a\) shifts the effective-speed scale, and \(b\) controls the local stretch. Because both scale factors are positive after exponentiation, the mapping from \(\tilde{\boldsymbol{v}}\) to \(\hat{\boldsymbol{v}}\) preserves wind-speed ordering. A formal proof is provided as follows:

\begin{proposition}[Monotonicity of the site-conditioned warp]
For a given wind farm with site features \(\boldsymbol{c}\), let
\([a,b]^\top=h_{\phi}(\boldsymbol{c})\in\mathbb{R}^{2}\), where \(h_{\phi}(\cdot)\) is the site-conditioned parameter head, \(a\) is the site-specific speed-scale parameter, and \(b\) is the site-specific stretch parameter. For any timestamp \(t\), let \(\tilde v_t\ge 0\) denote the equivalent wind speed and define the unclipped warp
\begin{equation}
    g(\tilde v_t)
    =
    \exp(a)(1+\tilde v_t)^{\exp(b)}-1.
\end{equation}
Then \(g(\tilde v_t)\) is strictly increasing with respect to \(\tilde v_t\). Consequently, the clipped effective wind speed $
    \hat v_t
    =
    \max(0,g(\tilde v_t))$
is non-decreasing with respect to \(\tilde v_t\).
\end{proposition}

\begin{proof}
For a fixed site, \(a\) and \(b\) are constants with respect to the equivalent wind speed \(\tilde v_t\). For \(\tilde v_t\ge 0\), the derivative of the unclipped warp is
\begin{equation}
    \frac{\partial g(\tilde v_t)}{\partial \tilde v_t}
    =
    \exp(a)\exp(b)(1+\tilde v_t)^{\exp(b)-1}.
\end{equation}
Since \(\exp(a)>0\), \(\exp(b)>0\), and \(1+\tilde v_t>0\), we have
\begin{equation}
    \frac{\partial g(\tilde v_t)}{\partial \tilde v_t}>0.
\end{equation}
Therefore, \(g(\tilde v_t)\) is strictly increasing on \(\tilde v_t\ge 0\).
It remains to consider the clipping operation. The function \(m(x)=\max(0,x)\) is non-decreasing in \(x\). Since \(g(\tilde v_t)\) is strictly increasing in \(\tilde v_t\), the composition \(m(g(\tilde v_t))\) is non-decreasing in \(\tilde v_t\). Thus, the clipped effective wind speed \(\hat v_t=\max(0,g(\tilde v_t))\) preserves wind-speed ordering: a larger equivalent wind speed cannot produce a smaller effective wind speed. The result applies element-wise to the vector form used in the main text.
\end{proof}

To further align the shared power curve with the target wind farm, UniWind learns site-conditioned anchors $[v_b,v_r]^\top \in \mathbb{R}^{2}$ using the same site-conditioning mechanism as $[a,b]^\top$. Here, \(v_b\) and \(v_r\) denote cut-in-like and rated-speed-like anchors, respectively. Specifically, the site-conditioned anchor head produces unconstrained scalars $\theta_b$ and $\theta_r$. To enforce their required order, we parameterize the anchors as follows:
\begin{equation}
\begin{aligned}
    v_b&=b_0+r_b\tanh(\theta_b),\\
    v_r&=v_b+\delta_{\min}+\operatorname{softplus}(\theta_r),
\end{aligned}
\end{equation}
where $b_0$ and $r_b$ control the admissible range of $v_b$, and $\delta_{\min}>0$ is a fixed positive margin. Since $\operatorname{softplus}(\theta_r)>0$, this parameterization gives $v_r-v_b=\delta_{\min}+\operatorname{softplus}(\theta_r)>0$ and therefore guarantees $v_r>v_b$. We then normalize the effective wind speed as follows:
\begin{equation}
   \boldsymbol{v}^*
    =
    \frac{
    \hat{\boldsymbol{v}}-v_b
    }{
    v_r-v_b
    }.
\end{equation}

\textbf{Shared physical power curve.} After site-conditioned speed calibration, we apply a shared, smooth power curve inspired by classical wind-turbine power curves~\cite{wang2019approaches}. The curve is not intended to reproduce a particular manufacturer specification; it captures the shared structure of wind generation, namely near-zero output before cut-in, nonlinear growth in the sub-rated region, and saturation near rated capacity. The physical prior is defined as:
\begin{equation}
    \boldsymbol{p}^{\mathrm{prior}}
    =
    P_{\mathrm{cap}} f_{\mathrm{curve}}(\boldsymbol{v}^*),
\end{equation}
where \(P_{\mathrm{cap}}\in \mathbb{R}\) is the rated capacity of the target wind farm and \(f_{\mathrm{curve}}(\cdot)\) maps the normalized effective wind speed \(\boldsymbol{v}^*\) to a capacity factor in \([0,1]\). 

The shared canonical curve \(f_{\mathrm{curve}}(\cdot)\) maps \(\boldsymbol{v}^*\) to a capacity factor in \([0,1]\) and follows the classical power-curve shape~\cite{wang2019approaches}: near-zero output before cut-in, nonlinear growth in the sub-rated region, and a rated-power plateau after the rated-speed region. The rising segment is motivated by the wind-energy relation \(P_{\mathrm{wind}}\propto \rho v^3\). We omit the high-wind cut-out regime from this available-power curve because abnormal shutdown and curtailment are modeled by the State-aware Power Corrector.

We parameterize this physical prior as an element-wise smooth approximation:
\begin{equation}
\begin{aligned}
    f_{\mathrm{curve}}(\boldsymbol{v}^*)
    &=
    \operatorname{clip}\left(
    g_{\mathrm{on}}(\boldsymbol{v}^*)
    g_{\mathrm{mid}}(\boldsymbol{v}^*)
    h(\boldsymbol{v}^*) \right.\\
    &\qquad\left.
    +g_{\mathrm{rated}}(\boldsymbol{v}^*),\,0,\,1
    \right),
\end{aligned}
\end{equation}
where \(\operatorname{clip}(\cdot,0,1)\) constrains the output to a valid capacity factor. The cut-in gate and rated-region gate are defined as:
\begin{equation}
\begin{aligned}
    g_{\mathrm{on}}(\boldsymbol{v}^*)
    &=
    \sigma(\tau_{\mathrm{on}}\boldsymbol{v}^*), \\
    g_{\mathrm{rated}}(\boldsymbol{v}^*)
    &=
    \sigma(\tau_{\mathrm{rated}}(\boldsymbol{v}^*-1)),
\end{aligned}
\end{equation}
where \(\sigma(\cdot)\) is the sigmoid function, \(\tau_{\mathrm{on}}>0\) controls the sharpness of the cut-in transition, and \(\tau_{\mathrm{rated}}>0\) controls the sharpness of the transition into the rated region. The sub-rated gate and rising segment are
\begin{equation}
\begin{aligned}
    g_{\mathrm{mid}}(\boldsymbol{v}^*)
    &=
    1-g_{\mathrm{rated}}(\boldsymbol{v}^*),\\
    \qquad
    h(\boldsymbol{v}^*)
    &=
    k\,\mathrm{softplus}(\boldsymbol{v}^*)^{n},
\end{aligned}
\end{equation}
where \(g_{\mathrm{mid}}(\cdot)\) suppresses the rising segment after the rated region, \(\mathrm{softplus}(x)=\log(1+\exp(x))\) provides a smooth positive basis for the rising segment, \(k>0\) controls the sub-rated scale, and \(n>0\) controls the growth curvature. We learn \(n\) rather than fixing the exponent to three because \(P_{\mathrm{wind}}\propto \rho v^3\) describes available wind energy, whereas the effective wind-farm power curve may deviate from an exact cubic law because of turbine efficiency. In implementation, \(n\) is bounded around the physically motivated cubic regime, while \(k\), \(\tau_{\mathrm{on}}\), and \(\tau_{\mathrm{rated}}\) are constrained to positive ranges. These constraints preserve the principal shape of a classical power curve while retaining site-adaptive flexibility.

\textbf{Physical upper-bound constraint.}
We further regularize the physical prior so that it behaves as a soft upper envelope of available generation. It should cover the observed power when generation is available, while remaining close to realized power during regular operation. We therefore define
\begin{equation}
\begin{aligned}
    \mathcal{L}_{\mathrm{prior}}
    &=
    \frac{1}{|\Omega|}
    \sum_{t\in\Omega}
    \left[
    \boldsymbol{p}_t-\boldsymbol{p}_t^{\mathrm{prior}}
    \right]_+ \\
    &+
    \lambda_{\mathrm{upper}}
    \frac{1}{|\Omega_{\mathrm{reg}}|}
    \sum_{t\in\Omega_{\mathrm{reg}}}
    \left[
    \boldsymbol{p}_t^{\mathrm{prior}}
    -
    \boldsymbol{p}_t
    -
    \epsilon_{\mathrm{slack}}
    \right]_+ .\\
\end{aligned}
\end{equation}
Here, \(\Omega\) is the set of observed time steps and \(\Omega_{\mathrm{reg}}\) is the subset assigned the rule-derived ``regular'' pseudo-label. We define \([x]_+=\max(x,0)\), and \(\epsilon_{\mathrm{slack}}\) specifies the allowed slack between the physical prior and observed power. The weight \(\lambda_{\mathrm{upper}}\) controls the tightness of the regular-operation constraint.

\subsection{Latent State Encoder}
The physical prior describes meteorology-driven available power, but realized generation can still deviate from it because of latent and time-varying operating conditions. The Latent State Encoder models this discrepancy by encoding historical prior--power deviations, retrieving state patterns under similar future meteorological conditions, and refining the retrieved representations into a temporally coherent state embedding sequence.

\textbf{Historical state encoder.} We represent recent latent operating behavior through the discrepancy between observed power and the physical prior. Let \(1{:}T_1\) denote the historical window and \(T_1{+}1{:}T_1+T_2\) denote the future window. Given the historical observed power \(\boldsymbol{x}\) and the historical physical prior \(\boldsymbol{p}^{\mathrm{prior}}_{1:T_1}\), which is detached from the computation graph, UniWind constructs discrepancy-aware historical state tokens:
\begin{equation}
    \boldsymbol{\psi}
    =
    \left[
    \boldsymbol{x},\;
    \boldsymbol{p}^{\mathrm{prior}}_{1:T_1},\;
    \boldsymbol{e},\;
    \Delta \boldsymbol{x},\;
    \Delta \boldsymbol{p}^{\mathrm{prior}}_{1:T_1},\;
    \Delta \boldsymbol{e},\;
    \boldsymbol{r},\;
    \Delta \boldsymbol{r}
    \right].
\end{equation}
Here, \(\boldsymbol{\psi}\in\mathbb{R}^{T_1\times 8}\) is the historical state-token sequence. The residual sequence \(\boldsymbol{e}=\boldsymbol{x}-\boldsymbol{p}^{\mathrm{prior}}_{1:T_1}\) measures the discrepancy between observed power and the learned physical prior.
% \(\boldsymbol{r}\) is the power-to-prior ratio with a denominator floor, and 
The power-to-prior ratio is defined as
\begin{equation}
    \boldsymbol{r}
    =
    \frac{\boldsymbol{x}}
    {\max\left(
        \boldsymbol{p}^{\mathrm{prior}}_{1:T_1},
        10^{-6}
    \right)},
\end{equation}
where the division and maximum operations are applied
element-wise, with $10^{-6}$ serving as the denominator
floor for numerical stability.
\(\Delta\) denotes the first-order temporal difference operator. Together, these features characterize
the magnitude and temporal dynamics of power deviations
from the available-power prior and are encoded into
$D$-dimensional historical state embeddings:
\begin{equation}
    \boldsymbol{s}^h
    =
    E_{\mathrm{hist}}(\boldsymbol{\psi}).
\end{equation}
Here, \(E_{\mathrm{hist}}(\cdot)\) consists of token normalization, an MLP projection, and a multi-scale temporal convolution block, while \(\boldsymbol{s}^h\in\mathbb{R}^{T_1\times D}\) denotes the resulting historical state embedding sequence.

\textbf{Weather encoder.} Future operational states are unavailable at forecast time, so future state retrieval must be conditioned on meteorological context that is available from the NWP sequence. UniWind first encodes the full weather sequence as follows:
\begin{equation}
\begin{aligned}
    \tilde{\boldsymbol{h}}^w
    =
    E_{\mathrm{w}}(\boldsymbol{w}),
    \quad
    \boldsymbol{h}^w
    =
    T_{\mathrm{w}}(\tilde{\boldsymbol{h}}^w).
\end{aligned}
\end{equation}
Here, \(\boldsymbol{w}\in\mathbb{R}^{(T_1+T_2)\times N_w}\) is the NWP sequence, \(E_{\mathrm{w}}(\cdot)\) is an MLP-based per-step weather encoder, and \(T_{\mathrm{w}}(\cdot)\) is a temporal weather encoder that incorporates positional and time information. The resulting \(\boldsymbol{h}^w\in\mathbb{R}^{(T_1+T_2)\times D}\) provides context-aware weather embeddings for both historical and future timestamps.

\textbf{Future latent state retrieval.} Since future operating states are unobserved, UniWind transfers state patterns from historical timestamps with similar meteorological contexts. It treats future weather embeddings as queries, historical weather embeddings as keys, and historical state embeddings as values:
\begin{equation}
    \tilde{\boldsymbol{s}}^f
    =
    \operatorname{Attn}
    \left(
    Q\boldsymbol{h}^w_{T_1+1:T_1+T_2},\;
    K\boldsymbol{h}^w_{1:T_1},\;
    V\boldsymbol{s}^h_{1:T_1}
    \right).
\end{equation}
Here, \(\tilde{\boldsymbol{s}}^f\in\mathbb{R}^{T_2\times D}\) denotes the retrieved future state sequence. The projections $Q$, $K$, and $V$ are learnable query, key, and value mappings, respectively, so the attention weights determine how historical state patterns are transferred to each future timestamp.

\textbf{Latent state refinement.} The retrieved future states provide weather-conditioned estimates, but they may be locally inconsistent when considered independently. UniWind therefore concatenates the historical and retrieved representations $[\boldsymbol{s}^h;\tilde{\boldsymbol{s}}^f]$ and refines them with a causal TCN-based refiner, which preserves the temporal direction of the forecasting process. The refined state features are then fused with the weather features through projection and an MLP, producing the final state embedding sequence \(\boldsymbol{z}\in\mathbb{R}^{(T_1+T_2)\times D}\) for the State-aware Power Corrector.

\subsection{State-aware Power Corrector}
\label{sec:state_aware_power_corrector}

The physical prior provides a meteorology-conditioned reference for power generation, whereas realized power may be reduced by latent operating conditions. The State-aware Power Corrector converts this prior into the final forecast by conditioning the correction on the inferred state embeddings, site context, and state-specific routing.

\textbf{Site-context conditioning.}
The same latent operating pattern can lead to different power deviations at different wind farms. To capture these persistent site differences, UniWind encodes the site features together with historical power statistics and injects the resulting context into the final state embeddings:
\begin{equation}
\begin{aligned}
    \boldsymbol{c}^{\mathrm{ctx}}
    =
    E_{\mathrm{ctx}}(\boldsymbol{c}, \boldsymbol{x}),
    \quad
    \tilde{\boldsymbol{z}}
    =
    \operatorname{LN}(\boldsymbol{z} + \boldsymbol{c}^{\mathrm{ctx}}).
\end{aligned}
\end{equation}
Here, \(E_{\mathrm{ctx}}(\cdot)\) is an MLP-based site-context encoder that fuses site features and historical power statistics. The resulting \(\boldsymbol{c}^{\mathrm{ctx}}\in\mathbb{R}^{D}\) is added to every time step of \(\boldsymbol{z}\), and layer normalization produces the site-aware latent state sequence \(\tilde{\boldsymbol{z}}\in\mathbb{R}^{(T_1+T_2)\times D}\). \(\operatorname{LN}(\cdot)\) denotes layer normalization.

\textbf{Pseudo-label-supervised state router.}
Different latent states require different correction strategies. UniWind therefore uses a state router to assign each site-aware state embedding to three state-specific experts, with rule-derived pseudo-labels providing weak supervision during training:
\begin{equation}
\begin{aligned}
    \hat{\boldsymbol{l}}
    =
    \operatorname{softmax}(R(\tilde{\boldsymbol{z}})),
    \quad
    \mathcal{L}_{\mathrm{router}}
    =
    \operatorname{CE}(\hat{\boldsymbol{l}}, \boldsymbol{l}).
\end{aligned}
\end{equation}
Here, \(R(\cdot)\) is an MLP-based state router, and \(\hat{\boldsymbol{l}}\in\mathbb{R}^{(T_1+T_2)\times 3}\) contains the routing probabilities for the shutdown, curtailment, and regular experts. \(\operatorname{CE}(\cdot,\cdot)\) denotes the cross-entropy loss. The supervision labels \(\boldsymbol{l}\), defined for the historical and future windows, are derived from wind-power behavior rather than ground-truth operational annotations. At inference time, the router uses \(\tilde{\boldsymbol{z}}\) to produce \(\hat{\boldsymbol{l}}\).

\textbf{State-specific correction.}
The router determines which correction mechanisms are relevant, while neighboring time steps may still share temporal context. UniWind therefore transforms \(\tilde{\boldsymbol{z}}\) with expert-specific MLPs, weights the resulting features by the routing probabilities. This produces the interaction-aware correction embedding \(\boldsymbol{z}'\in\mathbb{R}^{(T_1+T_2)\times D}\), from which UniWind predicts expert-specific scaling and upper-bound parameters:
\begin{equation}
\begin{aligned}
    (\boldsymbol{\alpha}^{\mathrm{sd}}, \boldsymbol{\beta}^{\mathrm{sd}})
    &=
    (\boldsymbol{0}, \boldsymbol{0}),\\
    \qquad
    (\boldsymbol{\alpha}^{\mathrm{cur}}, \boldsymbol{\beta}^{\mathrm{cur}})
    &=
    \left(
    \boldsymbol{1},
    \sigma(H_{\beta}(\boldsymbol{z}'))
    \right),\\
    \qquad
    (\boldsymbol{\alpha}^{\mathrm{reg}}, \boldsymbol{\beta}^{\mathrm{reg}})
    &=
    \left(
    \sigma(H_{\alpha}(\boldsymbol{z}')),
    \boldsymbol{1}
    \right),
\end{aligned}
\end{equation}
where \(H_{\alpha}(\cdot)\) and \(H_{\beta}(\cdot)\) denote MLP-based parameter heads. The shutdown expert $\mathrm{sd}$ sets both the scaling factor and upper bound to zero; the curtailment expert $\mathrm{cur}$ retains the prior scale but learns an upper-bound ratio; and the regular expert $\mathrm{reg}$ learns a multiplicative correction under the rated-capacity bound.

For each expert \(e\in\{\mathrm{sd},\mathrm{reg},\mathrm{cur}\}\), UniWind applies the predicted scaling factor and upper bound to the physical prior:
\begin{equation}
    \hat{\boldsymbol{p}}^{e}
    =
    \operatorname{clip}
    \left(
    \boldsymbol{\alpha}^{e}\odot\boldsymbol{p}^{\mathrm{prior}},
    0,
    \boldsymbol{\beta}^{e}P_{\mathrm{cap}}
    \right).
\end{equation}
The final forecast is the routing-weighted aggregation of the bounded expert outputs:
\begin{equation}
    \hat{\boldsymbol{p}}
    =
    \sum_{e\in\{\mathrm{sd},\mathrm{reg},\mathrm{cur}\}}
    \hat{\boldsymbol{l}}_{e}\odot \hat{\boldsymbol{p}}^{e},
\end{equation}
where \(\hat{\boldsymbol{p}}\in\mathbb{R}^{T_1+T_2}\) is the predicted power sequence and \(\hat{\boldsymbol{l}}_{e}\) denotes the routing probability for expert \(e\). Because the routing probabilities are normalized by the softmax, the final output combines the expert predictions while preserving their nonnegative and bounded forms.

\subsection{Training Objective and Inference}
\label{sec:training_objectives}

UniWind is trained to predict both the historical and future power sequences while preserving the physical interpretation of the prior. The pseudo-label-supervised state router further encourages the state-aware experts to specialize. At inference time, the same routed forward path is used without future realized power or rule-derived state labels.

\subsubsection{Training Loss}
\label{sec:training_loss}

Historical reconstruction encourages the learning of stable latent power-state representations, while future prediction directly supervises the day-ahead forecasting objective. Accordingly, the joint reconstruction and prediction loss is defined as
\begin{equation}
    \mathcal{L}_{\mathrm{mse}}
    =
    \frac{1}{T_1}
    \left\lVert
        \hat{\boldsymbol{x}}-\boldsymbol{x}
    \right\rVert_2^2
    +
    \frac{1}{T_{\mathrm{day}}}
    \left\lVert
        \hat{\boldsymbol{y}}^{\mathrm{DA}}
        -\boldsymbol{y}^{\mathrm{DA}}
    \right\rVert_2^2,
\end{equation}
where $\boldsymbol{x}$ and $\hat{\boldsymbol{x}}$ denote the observed and reconstructed historical power sequences, respectively. The day-ahead target and its corresponding prediction are given by $\boldsymbol{y}^{\mathrm{DA}} = \boldsymbol{y}_{T_{\mathrm{gap}}+1:T_2}$ and $\hat{\boldsymbol{y}}^{\mathrm{DA}} = \hat{\boldsymbol{y}}_{T_{\mathrm{gap}}+1:T_2}$, respectively.

The final training objective is
\begin{equation}
    \mathcal{L}
    =
    \mathcal{L}_{\mathrm{mse}}
    +
    \mathcal{L}_{\mathrm{prior}}
    +
    \lambda_{\mathrm{router}}\mathcal{L}_{\mathrm{router}},
\end{equation}
where \(\lambda_{\mathrm{router}}\) is the weight for the router-supervision loss.  The detailed training process is shown in Algorithm~\ref{alg:training}.

\begin{algorithm}[t]
\caption{Training procedure of UniWind}
\label{alg:training}
\begin{algorithmic}[1]
\REQUIRE Training minibatch containing $(\boldsymbol{w},\boldsymbol{x},\boldsymbol{y},\boldsymbol{c})$, raw physical inputs, rated capacity, and pseudo-labels $\boldsymbol{l}$; model parameters $\theta$
\ENSURE Updated model parameters $\theta$
\STATE Construct the site-calibrated physical prior $\boldsymbol{p}^{\mathrm{prior}}$ from raw wind speed, air density, site features, and rated capacity.
\STATE Encode historical prior--power discrepancies into $\boldsymbol{s}^{h}$ and encode the full NWP sequence into $\boldsymbol{h}^{w}$.
\STATE Retrieve future latent states $\tilde{\boldsymbol{s}}^{f}$ from historical states using future and historical weather embeddings.
\STATE Refine $[\boldsymbol{s}^{h};\tilde{\boldsymbol{s}}^{f}]$ and fuse the result with weather features to obtain $\boldsymbol{z}$.
\STATE Inject site context, compute router probabilities $\hat{\boldsymbol{l}}$, and produce the bounded expert predictions $\hat{\boldsymbol{p}}^{e}$.
\STATE Aggregate the experts to obtain $\hat{\boldsymbol{x}}$ and $\hat{\boldsymbol{y}}$; compute $\mathcal{L}_{\mathrm{mse}}$, $\mathcal{L}_{\mathrm{prior}}$, and $\mathcal{L}_{\mathrm{router}}$.
\STATE Compute $\mathcal{L}$ and update $\theta$ by back-propagation and the configured optimizer.
\end{algorithmic}
\end{algorithm}

\subsubsection{Inference Procedure}
\label{sec:inference_procedure}

At inference time, a forecast is issued once at the cutoff $\tau_0$ using only the admissible information specified in Section~\ref{sec:preliminaries}. In particular, the model receives the historical power sequence, the NWP sequence available at $\tau_0$, static site features, raw wind-speed and air-density inputs, and the rated capacity. Future realized power, future operational annotations, and pseudo-labels are not used.

The inference procedure follows the forward path of UniWind. First, the Physical Prior Estimator computes $\boldsymbol{p}^{\mathrm{prior}}$ from density-aware wind-speed normalization, site-conditioned calibration, and the shared physical power curve. Second, the historical power and detached historical prior are encoded into $\boldsymbol{s}^{h}$, while the full NWP sequence is transformed into $\boldsymbol{h}^{w}$. Future weather embeddings query the historical weather--state pairs to obtain $\tilde{\boldsymbol{s}}^{f}$, after which temporal refinement and weather--state fusion produce $\boldsymbol{z}$. Third, site context is injected to form $\tilde{\boldsymbol{z}}$, the router generates $\hat{\boldsymbol{l}}$, and the shutdown, curtailment, and regular experts produce bounded predictions that are aggregated according to these probabilities. The future segment $\hat{\boldsymbol{y}}$ is then returned. Under the day-ahead protocol, the evaluated delivery-day forecast is the final $T_\mathrm{day}$ steps of this segment.

UniWind produces the complete forecast in a single forward pass without teacher forcing or autoregressive feedback. Its computational cost depends on the input and forecast lengths, but it does not require sequential decoder calls for individual forecast steps.

\section{Experiments}
\subsection{Experimental Setup}
\label{sec:setup}

\subsubsection{Datasets and Data Processing}
We evaluate UniWind on 24 real-world datasets collected from distinct wind farms. These datasets include the Penmanshiel~\cite{plumley_penmanshiel_2022} and Kelmarsh~\cite{plumley_kelmarsh_2022} wind farms in the UK, while the remaining datasets span three provinces in China: Shandong, Shanxi, and Anhui. The corresponding NWP data are obtained from the European Centre for Medium-Range Weather Forecasts (ECMWF)\footnote{https://www.ecmwf.int/en/forecasts/datasets} and the Global Forecast System (GFS)\footnote{https://www.ncep.noaa.gov/products/global-forecast-system}.

We provide detailed information on the datasets used in this study. The wind power data were collected from 24 wind farms across four geographically distinct regions: Anhui, Shandong, and Shanxi in China, and the United Kingdom. Table~\ref{tab:wind_power_datasets} summarizes basic information about these wind power datasets, including their temporal coverage, temporal interval, and modeled capacity. Table~\ref{tab:nwp_datasets} reports the numerical weather prediction (NWP) datasets used as meteorological covariates. Since GFS data are unavailable for the UK datasets, only ECMWF data are used for this region. In addition, only five ECMWF features are available for the UK datasets: temperature2m, surfacePressure, windSpeed100m, windDirection100m, and airDensity. Accordingly, we use only these five features for training and evaluation on the UK datasets and adopt the same feature set in the zero-shot experiments to ensure input consistency.  For each dataset, we use the last two months for testing, the preceding two months for validation, and earlier data for training.

The raw data sources have different temporal resolutions. During preprocessing, we resample the wind power and NWP sequences by upsampling or downsampling as needed and align all variables to a unified 15-minute interval. We match each wind farm with the corresponding NWP grid point based on its latitude and longitude. For wind-farm datasets using two meteorological sources, we concatenate the feature dimensions from both sources as the weather sequence. To avoid information leakage among the training, validation, and test sets, we remove overlapping sequences around the split boundaries. All reported prediction results are computed after the model outputs are denormalized.

\begin{table*}[!h]
% \small
\centering
\caption{Basic information about the wind power datasets.}
\label{tab:wind_power_datasets}
\begin{tabular}{ccccc}
\toprule
\multicolumn{1}{c}{{Dataset}} & 
\multicolumn{1}{c}{Province} & 
\multicolumn{1}{c}{Temporal Duration} & 
\multicolumn{1}{c}{Temporal Interval} & 
\multicolumn{1}{c}{Capacity (MW)}\\
\midrule
 AH\_A & Anhui, China & 20230101-20241031 & 15 minutes & 149 \\
 AH\_B & Anhui, China & 20230101-20241031 & 15 minutes & 49.5 \\
 AH\_C & Anhui, China & 20230101-20241031 & 15 minutes & 49.5 \\
 AH\_D & Anhui, China & 20230101-20241031 & 15 minutes & 50 \\
 AH\_E & Anhui, China & 20230101-20241031 & 15 minutes & 64 \\
 AH\_F & Anhui, China & 20230101-20241031 & 15 minutes & 98 \\
 AH\_G & Anhui, China & 20230101-20241031 & 15 minutes & 100 \\
 AH\_H & Anhui, China & 20230101-20241031 & 15 minutes & 50 \\
 SD\_A & Shandong, China & 20240522-20251110 & 15 minutes & 147.75 \\
 SD\_B & Shandong, China & 20240522-20251110 & 15 minutes & 99 \\
 SD\_C & Shandong, China & 20240522-20251110 & 15 minutes & 250 \\
 SD\_D & Shandong, China & 20240522-20251110 & 15 minutes & 603.5 \\
 SD\_E & Shandong, China & 20240522-20251110 & 15 minutes & 198 \\
 SD\_F & Shandong, China & 20240522-20251110 & 15 minutes & 154 \\
 SD\_G & Shandong, China & 20240522-20251110 & 15 minutes & 99.5 \\
 SX\_A & Shanxi, China & 20241202-20251214 & 15 minutes & 95 \\
 SX\_B & Shanxi, China & 20241202-20251214 & 15 minutes & 122 \\
 SX\_C & Shanxi, China & 20241202-20251214 & 15 minutes & 47 \\
 SX\_D & Shanxi, China & 20241202-20251214 & 15 minutes & 137 \\
 SX\_E & Shanxi, China & 20241202-20251214 & 15 minutes & 20 \\
 SX\_F & Shanxi, China & 20241202-20251214 & 15 minutes & 24 \\
 SX\_G & Shanxi, China & 20241202-20251214 & 15 minutes & 11 \\
 Kelmarsh & Northamptonshire, United Kingdom & 20170925-20210228 & 10 minutes & 12.3 \\
 Penmanshiel & Scottish Borders, United Kingdom & 20170925-20210228 & 10 minutes & 28.7 \\
\bottomrule
\end{tabular}
\end{table*}

\begin{table*}[!h]
\small
\centering
\caption{Basic information about the NWP datasets.}
\label{tab:nwp_datasets}
\begin{tabularx}{\textwidth}{cc>{\raggedright\arraybackslash}X}
\toprule
Dataset & Temporal Interval & Features \\
\midrule
ECMWF & One hour & Temperature2m, relativeHumidity2m, windSpeed10m, windDirection10m, windSpeed100m, windDirection100m, pressure\_msl, surfacePressure, airDensity \\
GFS & One hour & Temperature2m, relativeHumidity2m, windSpeed10m, windDirection10m, windSpeed80m, windDirection80m, windSpeed120m, windDirection120m, pressure\_msl, surfacePressure, airDensity \\
\bottomrule
\end{tabularx}
\end{table*}

\subsubsection{Baselines}
We compare UniWind with representative and state-of-the-art models for day-ahead wind power forecasting, which are categorized into five groups:
\textbf{Physical models.} We include PowerCurve~\cite{haas_windpowerlib_2024}, a physics-based wind-turbine power-curve model that converts meteorological conditions into power estimates using predefined power-curve assumptions.
\textbf{Statistical models.} We consider widely used tree-based statistical learning models, including LightGBM~\cite{ju2019model} and XGBoost~\cite{phan2021hybrid}, which learn nonlinear mappings from meteorological covariates to wind power.
\textbf{Time-series forecasting models.} We compare UniWind with competitive long-horizon time-series forecasting models, including PatchTST~\cite{nietime}, iTransformer~\cite{liuitransformer}, TimeMixer~\cite{wangtimemixer}, and xPatch~\cite{stitsyuk2025xpatch}.
\textbf{Renewable-energy forecasting models.} We further include solar power forecasting models, including CrossViViT~\cite{boussif2023improving}, FusionSF~\cite{ma2024fusionsf}, and the wind power forecasting model 2DXformer~\cite{zhang20252dxformer}.
\textbf{Foundation models.}
To evaluate zero-shot generalization, we compare UniWind with the wind-specific foundation model WindFM~\cite{fan2025windfm} and general time-series foundation models, including Chronos-2~\cite{ansari2025chronos} and Moirai~\cite{liumoirai}. For a fair comparison, Chronos-2 and Moirai are adapted to use the same inputs as UniWind.

\subsubsection{Baseline Details}
We provide detailed descriptions of the baseline models used in our experiments. These baselines cover physical modeling, statistical learning, time-series forecasting, renewable-energy forecasting, and foundation models.

\begin{table*}[!h]
\small
\centering
\caption{Input information used by baselines.}
\label{tab:baseline_inputs}
\begin{tabular}{lccc}
\toprule
\textbf{Model} & \textbf{Historical NWP} & \textbf{Future NWP} & \textbf{Historical Power} \\
\midrule
PowerCurve~\cite{haas_windpowerlib_2024} & -- & \(\surd\) & -- \\
LightGBM~\cite{ju2019model} & -- & \(\surd\) & \(\surd\) \\
XGBoost~\cite{phan2021hybrid} & -- & \(\surd\) & \(\surd\) \\
% \midrule
PatchTST~\cite{nietime} & \(\surd\) & -- & \(\surd\) \\
iTransformer~\cite{liuitransformer} & \(\surd\)  & -- & \(\surd\) \\
TimeMixer~\cite{wangtimemixer} & \(\surd\) & -- & \(\surd\) \\
xPatch~\cite{stitsyuk2025xpatch} & \(\surd\) & -- & \(\surd\) \\

WindFM~\cite{fan2025windfm} & \(\surd\) & -- & \(\surd\) \\
CrossViViT~\cite{boussif2023improving} & \(\surd\) & \(\surd\) & \(\surd\) \\
2DXformer~\cite{zhang20252dxformer} & \(\surd\) & \(\surd\) & \(\surd\) \\
FusionSF~\cite{ma2024fusionsf} & \(\surd\) & \(\surd\)  & \(\surd\) \\
% Chronos-2~\cite{ansari2025chronos} & -- & -- & \(\surd\) \\
Chronos-2~\cite{ansari2025chronos}  & \(\surd\) & \(\surd\) & \(\surd\) \\
% Moirai~\cite{liumoirai} & -- & -- & \(\surd\) \\
% Moirai-PT~\cite{liumoirai} & -- & -- & \(\surd\) \\
Moirai~\cite{liumoirai} & \(\surd\) & \(\surd\) & \(\surd\) \\
% \midrule
\bottomrule
\end{tabular}
\end{table*}

\begin{itemize}
    \item \textbf{PowerCurve}~\cite{haas_windpowerlib_2024}: A physics-based wind power model that converts meteorological conditions into power estimates according to predefined wind-turbine power-curve assumptions.

    \item \textbf{LightGBM}~\cite{ju2019model}: A gradient boosting decision tree model that learns nonlinear mappings from meteorological covariates to wind power generation. It is widely used as a strong statistical baseline for tabular forecasting tasks. We concatenate future NWP and historical power as input to LightGBM.

    \item \textbf{XGBoost}~\cite{phan2021hybrid}: An efficient tree-based boosting model that captures nonlinear feature interactions from future NWP variables and provides a strong practical baseline for wind power forecasting.  We use future NWP data and historical power data as inputs to XGBoost.

    \item \textbf{PatchTST}~\cite{nietime}: A patch-based Transformer model for long-term time-series forecasting. It segments historical power sequences into patches and models temporal dependencies with a channel-independent design.

    \item \textbf{iTransformer}~\cite{liuitransformer}: A Transformer-based time-series forecasting model that treats variables as tokens and learns multivariate temporal dependencies through an inverted attention structure.

    \item \textbf{TimeMixer}~\cite{wangtimemixer}: A time-series forecasting model that mixes temporal patterns across multiple scales, enabling it to capture both short-term and long-term dependencies from historical power observations.

    \item \textbf{xPatch}~\cite{stitsyuk2025xpatch}: A long-horizon forecasting model that enhances patch-based time-series representations for historical sequence modeling.

    \item \textbf{WindFM}~\cite{fan2025windfm}: A wind power foundation model pretrained on large-scale wind power data. We evaluate it in the zero-shot setting using the released checkpoint.

    \item \textbf{CrossViViT}~\cite{boussif2023improving}: A renewable-energy forecasting model originally designed to fuse spatio-temporal weather context with station-level time series. In our experiments, this model is adapted for wind power prediction by incorporating historical and future NWP data as well as historical generation data.

    \item \textbf{FusionSF}~\cite{ma2024fusionsf}: A solar power forecasting model that fuses heterogeneous temporal and weather information. Since satellite cloud imagery is unavailable, we replace the satellite-image branch with NWP data.

    \item \textbf{2DXformer}~\cite{zhang20252dxformer}: A wind power forecasting model that captures temporal dependencies from historical weather and power sequences through Transformer-based representation learning. In our experiments, future NWP data are additionally incorporated to ensure consistency across the experimental datasets.

    \item \textbf{Chronos-2}~\cite{ansari2025chronos}: A general-purpose time-series foundation model pretrained on heterogeneous time-series corpora. In our experiments, NWP data are incorporated to ensure consistency across the experimental datasets.
    
    \item \textbf{Moirai}~\cite{liumoirai}: A universal time-series foundation model for zero-shot forecasting across diverse domains. In our experiments, NWP data are incorporated to ensure consistency across the experimental datasets.

\end{itemize}

Table~\ref{tab:baseline_inputs} summarizes the input information used by each baseline. Here, \(\surd\) indicates that the corresponding input type is used by the model, and ``--'' indicates that it is not used.

\subsubsection{Evaluation Metrics and Protocols}
\textbf{Setup.}
We evaluate each model on the final 24-hour delivery interval using the denormalized power values. The historical and future input windows contain $T_1=96$ and $T_2=192$ time steps, respectively. Because the data are sampled at 15-minute intervals, the first $T_{\mathrm{gap}}=96$ future steps span the 24-hour lead-time interval following forecast issuance, while the final $H=96$ steps constitute the subsequent 24-hour delivery interval. This protocol enforces the same information boundary for all methods and reflects the operational requirement to issue the forecast before the target delivery day begins.

In the full-shot setting, each model is trained separately on the training split of each wind-farm dataset and evaluated on its held-out test split. We report the results on SD\_A, SX\_A, AH\_A, and Penmanshiel in the main comparison table, while the remaining 20 farms are reported in Table~\ref{tab:additional_fullshot_transposed}. In the cross-farm zero-shot setting, the target farms are excluded from parameter fitting: a single checkpoint for each of Chronos-2, Moirai, and UniWind is trained using all datasets except the four held-out target farms and then evaluated on those targets. In the cross-region zero-shot setting, a separate checkpoint is trained for each target farm using datasets from regions other than the target region. WindFM is evaluated using its publicly released checkpoint. Thus, zero-shot evaluation excludes target-farm training labels while retaining the admissible historical power, future NWP, and site information available at forecast issuance. Each trainable model is run with five random seeds, and we report the mean $\pm$ standard deviation. PowerCurve is deterministic, and WindFM uses a released checkpoint. Therefore, no standard deviation is reported for these two methods.

\textbf{Metrics.}
Let $N_d$ denote the number of test windows for wind farm $d$, let $H=96$ denote the evaluated forecast horizon, and let $y_{i,t}^{(d)}$ and $\hat{y}_{i,t}^{(d)}$ denote the ground-truth and predicted power at timestamp $t$ in test window $i$. We use Mean Absolute Error (MAE) and Root Mean Squared Error (RMSE) as the primary metrics. For each wind farm, they are defined as
\begin{equation}
    \mathrm{MAE}_d
    =
    \frac{1}{N_d H}
    \sum_{i=1}^{N_d}\sum_{t=1}^{H}
    \left|\hat{y}_{i,t}^{(d)}-y_{i,t}^{(d)}\right|,
\end{equation}
\begin{equation}
    \mathrm{RMSE}_d
    =
    \sqrt{
    \frac{1}{N_d H}
    \sum_{i=1}^{N_d}\sum_{t=1}^{H}
    \left(\hat{y}_{i,t}^{(d)}-y_{i,t}^{(d)}\right)^2
    }.
\end{equation}
Both MAE and RMSE are computed in the original power unit (MW), and lower values indicate better forecasting accuracy. MAE measures the average absolute deviation, whereas RMSE assigns a larger penalty to occasional large errors and is therefore particularly informative for abrupt power changes and abnormal operating periods.

Because wind farms have different rated capacities, we additionally use capacity-normalized metrics for regional aggregation. Let $P_{\mathrm{cap}}^{(d)}$ be the rated capacity of farm $d$. The normalized errors for an individual farm are
\begin{equation}
\begin{aligned}
    \mathrm{NMAE}_d
    &=100\times\frac{\mathrm{MAE}_d}{P_{\mathrm{cap}}^{(d)}},\\
    \mathrm{NRMSE}_d
    &=100\times\frac{\mathrm{RMSE}_d}{P_{\mathrm{cap}}^{(d)}}.
\end{aligned}
\end{equation}
For a region $r$ containing the set of farms $\mathcal{D}_r$, the reported regional values are the unweighted averages of the corresponding farm-level normalized errors:
\begin{equation}
\begin{aligned}
    \mathrm{NMAE}_r
    &=\frac{1}{|\mathcal{D}_r|}
    \sum_{d\in\mathcal{D}_r}\mathrm{NMAE}_d,\\
    \mathrm{NRMSE}_r
    &=\frac{1}{|\mathcal{D}_r|}
    \sum_{d\in\mathcal{D}_r}\mathrm{NRMSE}_d.
    \end{aligned}
\end{equation}
The improvement percentages shown in the comparison tables are computed relative to the strongest competing method for the same dataset, metric, and evaluation setting:
\begin{equation}
    \mathrm{Improvement}(\%)
    =100\times
    \frac{M_{\mathrm{ref}}-M_{\mathrm{UniWind}}}{M_{\mathrm{ref}}},
\end{equation}
where $M$ denotes MAE or RMSE and $M_{\mathrm{ref}}$ is the lowest value among the competing baselines.

\begin{table*}[htbp]
\centering
\small
\caption{Prediction performance comparison on four datasets in terms of MAE and RMSE.}
\label{tab:baseline}

\begingroup
\setlength{\tabcolsep}{3pt}
\renewcommand{\arraystretch}{1.3}
\resizebox{\textwidth}{!}{%
\begin{tabular}{llcccccccc}
\toprule
& \multirow{2}{*}{Method}
& \multicolumn{2}{c}{SD\_A}
& \multicolumn{2}{c}{SX\_A}
& \multicolumn{2}{c}{AH\_A}
& \multicolumn{2}{c}{Penmanshiel} \\
\cmidrule(lr){3-4} \cmidrule(lr){5-6} \cmidrule(lr){7-8} \cmidrule(lr){9-10}
& & MAE & RMSE & MAE & RMSE & MAE & RMSE & MAE & RMSE \\
\midrule
\multirow{12}{*}{Full-shot}
& PowerCurve~\cite{haas_windpowerlib_2024} & 30.22 & 43.89 & 23.70 & 31.14 & 57.67 & 67.53 & 10.06 & 12.30 \\
\cline{2-10}
& LightGBM~\cite{ju2019model} & 15.17$\pm$0.10 & 24.31$\pm$0.24 & \underline{12.06$\pm$0.82} & \underline{16.01$\pm$0.50} & 27.13$\pm$0.20 & 34.74$\pm$0.24 & 6.80$\pm$0.03 & \underline{8.08$\pm$0.03} \\
& XGBoost~\cite{phan2021hybrid} & \underline{14.81$\pm$0.08} & \underline{23.68$\pm$0.24} & 15.39$\pm$0.02 & 18.59$\pm$0.03 & 26.68$\pm$0.44 & 34.41$\pm$0.49 & 6.79$\pm$0.00 & 8.09$\pm$0.00 \\
\cline{2-10}
& PatchTST~\cite{nietime} & 33.19$\pm$0.72 & 46.85$\pm$0.57 & 20.61$\pm$0.25 & 32.34$\pm$0.89 & 41.58$\pm$0.71 & 54.80$\pm$0.62 & 8.54$\pm$0.41 & 10.86$\pm$0.45 \\
& iTransformer~\cite{liuitransformer} & 36.70$\pm$1.10 & 50.26$\pm$1.40 & 19.72$\pm$0.85 & 31.89$\pm$1.09 & 42.67$\pm$1.57 & 55.67$\pm$1.47 & 8.45$\pm$0.21 & 10.77$\pm$0.18 \\
& TimeMixer~\cite{wangtimemixer} & 34.58$\pm$0.25 & 46.75$\pm$0.37 & 20.10$\pm$1.82 & 30.94$\pm$1.31 & 43.74$\pm$0.54 & 56.50$\pm$1.05 & 8.11$\pm$0.36 & 10.24$\pm$0.33 \\
& xPatch~\cite{stitsyuk2025xpatch} & 33.22$\pm$0.50 & 45.61$\pm$1.07 & 20.88$\pm$1.12 & 32.17$\pm$1.13 & 45.36$\pm$0.62 & 58.25$\pm$1.10 & 9.50$\pm$0.89 & 11.78$\pm$1.04 \\
\cline{2-10}
& CrossViViT~\cite{boussif2023improving} & 15.85$\pm$0.65 & 24.49$\pm$1.43 & 12.64$\pm$2.09 & 20.46$\pm$2.78 & \underline{13.34$\pm$0.81} & \underline{18.79$\pm$1.51} & \underline{6.65$\pm$0.35} & 9.15$\pm$0.36 \\
& FusionSF~\cite{ma2024fusionsf} & 23.39$\pm$0.38 & 33.46$\pm$0.61 & 13.79$\pm$0.39 & 21.69$\pm$1.78 & 19.70$\pm$1.69 & 26.93$\pm$1.86 & 7.30$\pm$0.24 & 9.62$\pm$0.31 \\
& 2DXformer~\cite{zhang20252dxformer} & 23.66$\pm$1.55 & 33.36$\pm$2.12 & 12.83$\pm$1.96 & 19.36$\pm$2.32 & 25.34$\pm$1.62 & 33.86$\pm$2.32 & 7.93$\pm$0.81 & 10.20$\pm$0.93 \\
\cline{2-10}
& \textbf{UniWind} & \textbf{11.96$\pm$0.17} & \textbf{17.79$\pm$0.15} & \textbf{10.28$\pm$0.16} & \textbf{15.37$\pm$0.43} & \textbf{12.26$\pm$0.29} & \textbf{17.39$\pm$0.31} & \textbf{6.04$\pm$0.07} & \textbf{7.80$\pm$0.01} \\
\cline{2-10}
& Improvement (\%) & 19.24 & 24.87 & 14.76 & 4.00 & 8.10 & 7.45 & 9.17 & 3.47 \\
\midrule
\midrule
\multirow{5}{*}{\shortstack{Cross-farm\\zero-shot}}
& WindFM~\cite{fan2025windfm} & 38.76 & 51.65 & 20.70 & 33.22 & 47.88 & 57.88 & 8.87 & 11.71 \\
& Moirai~\cite{liumoirai} & \underline{22.91$\pm$0.13} & \underline{35.05$\pm$0.25} & 20.45$\pm$0.29 & 28.76$\pm$0.64 & \underline{28.88$\pm$0.40} & \underline{40.85$\pm$0.58} & \underline{7.47$\pm$0.04} & \underline{10.20$\pm$0.09} \\
& Chronos-2~\cite{ansari2025chronos} & 26.33$\pm$0.18 & 38.86$\pm$0.26 & \underline{17.00$\pm$0.22} & \underline{24.70$\pm$0.36} & 34.42$\pm$0.38 & 48.08$\pm$0.72 & 10.05$\pm$0.09 & 14.58$\pm$0.07 \\
\cline{2-10}
& \textbf{UniWind} & \textbf{14.43$\pm$0.24} & \textbf{22.26$\pm$0.27} & \textbf{12.43$\pm$0.13} & \textbf{18.71$\pm$0.19} & \textbf{15.05$\pm$0.16} & \textbf{19.12$\pm$0.43} & \textbf{6.56$\pm$0.09} & \textbf{8.71$\pm$0.09} \\
\cline{2-10}
& Improvement (\%) & 37.01 & 36.49 & 26.88 & 24.25 & 47.89 & 53.19 & 12.18 & 14.61 \\
\midrule
\midrule
\multirow{5}{*}{\shortstack{Cross-region\\zero-shot}}
& WindFM~\cite{fan2025windfm} & 38.76 & 51.65 & 20.70 & 33.22 & 47.88 & 57.88 & 8.87 & 11.71 \\
& Moirai~\cite{liumoirai} & \underline{23.83$\pm$0.26} & \underline{35.69$\pm$0.44} & 21.16$\pm$0.53 & 28.83$\pm$0.74 & \underline{30.02$\pm$0.87} & \underline{41.59$\pm$0.60} & \underline{7.60$\pm$0.07} & \underline{10.26$\pm$0.12} \\
& Chronos-2~\cite{ansari2025chronos} & 28.44$\pm$0.40 & 40.38$\pm$0.84 & \underline{16.40$\pm$0.30} & \underline{23.50$\pm$0.35} & 37.57$\pm$0.44 & 48.81$\pm$0.57 & 8.45$\pm$0.05 & 11.46$\pm$0.14 \\
\cline{2-10}
& \textbf{UniWind} & \textbf{20.15$\pm$0.20} & \textbf{25.84$\pm$0.28} & \textbf{13.88$\pm$0.04} & \textbf{19.81$\pm$0.17} & \textbf{22.17$\pm$0.20} & \textbf{28.46$\pm$0.16} & \textbf{6.96$\pm$0.08} & \textbf{9.15$\pm$0.12} \\
\cline{2-10}
& Improvement (\%) & 15.44 & 27.60 & 15.37 & 15.70 & 26.15 & 31.57 & 8.42 & 10.82 \\
\bottomrule
\end{tabular}%
}
\endgroup

\parbox{\linewidth}{%
\scriptsize
\textit{Note: PowerCurve is deterministic, while WindFM is evaluated using a released checkpoint. Therefore, no standard deviations are reported.}
}
\end{table*}

\subsubsection{Task Compliance and Evaluation Settings}

UniWind obeys the above information boundary while separating two sources of variation in realized power. The Physical Prior Estimator uses future NWP and site features to estimate the meteorology-driven available-power prior $\boldsymbol{p}^{\mathrm{prior}}$. This prior is not treated as the final delivery forecast, because actual generation can be reduced by latent operating conditions. The Latent State Encoder uses historical power, the historical prior, and future NWP context to retrieve and refine state representations. The State-aware Power Corrector then converts the prior into the final forecast $\hat{\boldsymbol{y}}$ through bounded, state-specific experts. At inference time, the complete prediction is produced from $(\boldsymbol{w},\boldsymbol{x},
\boldsymbol{c})$ only; the target power sequence is never required as an input.

The same task definition applies to all three evaluation settings in this paper. In the full-shot setting, the parameters are trained separately for each wind farm using its training split. In both cross-farm and cross-region zero-shot settings, the model is trained without using data from the target wind farm, and no target-specific fine-tuning is performed. At inference time, only the target farm's inputs $(\boldsymbol{w},\boldsymbol{x},\boldsymbol{c})$ available at the forecast issuance time are used. Therefore, zero-shot refers to the absence of target-farm training, not to the absence of historical observations or site-specific NWP at test time. These settings share the same forecasting target and information boundary but differ in their training data and, for the Chinese datasets, their NWP feature sets.

\subsubsection{Implementation Details}
%\textbf{Implementation details.}
UniWind is implemented in PyTorch, and all experiments are conducted on an NVIDIA GeForce RTX 2080 Ti 11 GB GPU. All baselines use the parameter settings reported in their respective papers. For UniWind, we use Adam with an initial learning rate of \(1 \times 10^{-4}\), a weight decay of \(1 \times 10^{-5}\), and a ReduceLROnPlateau scheduler with a decay factor of 0.5 and a patience of 5.  \(\epsilon_{\mathrm{slack}}\), \(\lambda_{\mathrm{upper}}\), and \(\lambda_{\mathrm{router}}\) are set to 1, 0.01, and 0.4, respectively. The hidden dimension \(D\) is 256. Both the weather encoder and state retriever use 4 attention heads, and the batch size is 32.

\subsection{Overall Performance}

Table~\ref{tab:baseline} presents the overall performance on the four representative wind farms. It reports full-shot comparisons with baseline models and zero-shot comparisons with foundation models. The best results are shown in \textbf{bold}, and the second-best results are \underline{underline}. 

\subsubsection{Full-shot Performance}

Table~\ref{tab:baseline} shows that UniWind achieves the best full-shot performance on all four representative datasets and for both metrics. Compared with the strongest baseline on each dataset, UniWind reduces MAE by 19.24\%, 14.76\%, 8.10\%, and 9.17\% on SD\_A, SX\_A, AH\_A, and Penmanshiel, respectively. The corresponding RMSE reductions are 24.87\%, 4.00\%, 7.45\%, and 3.47\%. The consistent improvements in both MAE and RMSE indicate that the gains are not limited to average pointwise accuracy; they also reduce larger forecasting errors, particularly on SD\_A and AH\_A.

The comparison across baseline families reveals a clear difference between a fixed physical mapping, unconstrained empirical learning, and the proposed structured decomposition. PowerCurve is less accurate than the strongest learned baselines, showing that a predefined wind-speed-to-power relationship cannot capture site-specific responses or latent operational deviations. LightGBM and XGBoost are strong statistical baselines, with LightGBM performing best among the baselines on SX\_A and XGBoost performing best on SD\_A. However, their relative performance varies across sites, which is consistent with learning local weather--power correlations without an explicit representation of available power and operating state.

The historical-sequence models PatchTST, iTransformer, TimeMixer, and xPatch generally perform worse than the models that use future NWP information. This gap is especially pronounced on SD\_A and AH\_A, suggesting that historical power patterns alone cannot fully recover the meteorological drivers of the next-day delivery interval. Models designed for renewable-energy forecasting, including CrossViViT, FusionSF, and 2DXformer, are more competitive because they incorporate future meteorological context. CrossViViT is the strongest baseline on AH\_A and obtains the lowest MAE on Penmanshiel, whereas LightGBM obtains the lowest RMSE on Penmanshiel and remains competitive on SX\_A. Nevertheless, UniWind outperforms these baselines despite using the same future NWP information as the renewable-energy models. This result suggests that the improvement is not obtained merely by adding future weather covariates, but by using them to construct a site-calibrated physical prior and then correcting that prior according to inferred latent states.

Overall, the full-shot results support the effectiveness of combining the Physical Prior Estimator with the State-aware Power Corrector: the former supplies a physically meaningful reference for meteorology-driven available power, while the latter accounts for deviations caused by hidden and time-varying operating conditions.

\subsubsection{Zero-shot Performance}
Table~\ref{tab:baseline} evaluates UniWind under two increasingly challenging transfer settings. In both settings, the target-farm power labels are excluded from parameter fitting, and UniWind is compared with WindFM, Moirai, and Chronos-2. UniWind achieves the best MAE and RMSE on all four target datasets in both settings, while the strongest competing foundation model varies with the target dataset.

\textbf{Cross-farm zero-shot.}
For this setting, a single checkpoint is trained using all datasets except the four target farms, and the same checkpoint is evaluated on SD\_A, SX\_A, AH\_A, and Penmanshiel. UniWind obtains MAEs of 14.43, 12.43, 15.05, and 6.56 MW, respectively, and the corresponding RMSEs are 22.26, 18.71, 19.12, and 8.71 MW. It improves over the strongest foundation-model baseline by 37.01\%/36.49\% on SD\_A, 26.88\%/24.25\% on SX\_A, 47.89\%/53.19\% on AH\_A, and 12.18\%/14.61\% on Penmanshiel, where each pair reports MAE/RMSE improvement. Moirai is the strongest baseline on SD\_A, AH\_A, and Penmanshiel, whereas Chronos-2 performs best on SX\_A. WindFM is less accurate than Moirai and UniWind on all four datasets, but outperforms Chronos-2 on Penmanshiel. The large gains on AH\_A and SD\_A show that a model trained across farms can transfer more effectively when it explicitly separates a transferable meteorology-driven prior from site- and state-dependent correction.

\begin{table*}[htbp]
\small
\centering
\caption{Additional full-shot performance on the remaining 20 wind farms.}
\label{tab:additional_fullshot_transposed}
\renewcommand{\arraystretch}{1.3}
\setlength{\tabcolsep}{1.2pt}
\resizebox{\textwidth}{!}{%
\begin{tabular}{llccccccccccc}
\toprule
Dataset & Metric
& PowerCurve~
& LightGBM
& XGBoost
& PatchTST
& iTransformer
& TimeMixer
& xPatch
& CrossViViT
& FusionSF
& 2DXformer
& \textbf{UniWind} \\
\midrule

\multirow{2}{*}{SD\_B}
& MAE
& 22.18
& 13.22$\pm$0.20
& 12.93$\pm$0.16
& 28.22$\pm$0.48
& 27.53$\pm$0.69
& 26.16$\pm$0.71
& 27.49$\pm$1.15
& \underline{12.61$\pm$0.40}
& 18.48$\pm$0.40
& 15.83$\pm$0.85
& \textbf{9.13$\pm$0.17} \\
& RMSE
& 32.19
& 20.10$\pm$0.35
& \underline{19.57$\pm$0.21}
& 38.89$\pm$0.30
& 37.54$\pm$1.37
& 35.12$\pm$0.67
& 36.82$\pm$1.29
& 19.86$\pm$0.53
& 27.14$\pm$0.60
& 23.31$\pm$1.06
& \textbf{14.59$\pm$0.23} \\

\multirow{2}{*}{SD\_C}
& MAE
& 50.60
& \underline{27.27$\pm$0.35}
& 27.55$\pm$0.30
& 50.25$\pm$1.00
& 55.35$\pm$1.15
& 54.83$\pm$0.45
& 51.05$\pm$3.62
& 27.60$\pm$0.50
& 37.67$\pm$1.20
& 38.33$\pm$1.45
& \textbf{21.35$\pm$0.22} \\
& RMSE
& 66.88
& \underline{38.85$\pm$0.28}
& 38.90$\pm$0.25
& 67.93$\pm$1.38
& 72.20$\pm$1.47
& 71.83$\pm$0.72
& 66.70$\pm$4.85
& 39.88$\pm$0.80
& 52.45$\pm$1.47
& 52.53$\pm$2.20
& \textbf{32.27$\pm$0.22} \\

\multirow{2}{*}{SD\_D}
& MAE
& 147.07
& 72.78$\pm$2.05
& \underline{70.49$\pm$0.78}
& 152.26$\pm$4.65
& 154.92$\pm$2.60
& 155.04$\pm$2.35
& 147.62$\pm$5.13
& 75.74$\pm$2.47
& 96.86$\pm$4.65
& 94.21$\pm$6.34
& \textbf{66.81$\pm$2.17} \\
& RMSE
& 195.23
& 99.52$\pm$1.33
& \textbf{98.31$\pm$1.09}
& 205.61$\pm$3.32
& 201.51$\pm$1.45
& 196.50$\pm$4.65
& 195.35$\pm$2.11
& 112.67$\pm$4.47
& 134.22$\pm$6.94
& 130.42$\pm$9.96
& \underline{98.43$\pm$0.84} \\

\multirow{2}{*}{SD\_E}
& MAE
& 43.40
& 25.50$\pm$0.26
& \underline{25.32$\pm$0.28}
& 43.18$\pm$0.99
& 45.36$\pm$2.28
& 42.61$\pm$0.63
& 42.27$\pm$1.45
& 26.93$\pm$0.83
& 34.71$\pm$1.37
& 32.97$\pm$1.66
& \textbf{20.69$\pm$0.16} \\
& RMSE
& 55.88
& 35.74$\pm$0.36
& \underline{35.58$\pm$0.38}
& 61.62$\pm$1.25
& 60.63$\pm$2.81
& 59.34$\pm$0.83
& 58.65$\pm$1.54
& 38.73$\pm$1.50
& 49.82$\pm$1.92
& 47.04$\pm$2.04
& \textbf{30.99$\pm$0.22} \\

\multirow{2}{*}{SD\_F}
& MAE
& 39.76
& 22.76$\pm$0.15
& 22.73$\pm$0.29
& 44.84$\pm$2.05
& 41.16$\pm$0.46
& 42.35$\pm$0.57
& 47.36$\pm$3.43
& \underline{22.10$\pm$0.57}
& 31.79$\pm$0.79
& 31.42$\pm$1.34
& \textbf{18.09$\pm$0.18} \\
& RMSE
& 52.62
& 33.03$\pm$0.23
& 33.14$\pm$0.26
& 59.49$\pm$2.43
& 54.44$\pm$1.06
& 55.56$\pm$1.26
& 61.46$\pm$4.16
& \underline{32.56$\pm$0.66}
& 44.91$\pm$1.00
& 42.44$\pm$1.12
& \textbf{27.50$\pm$0.28} \\

\multirow{2}{*}{SD\_G}
& MAE
& 25.81
& 13.63$\pm$0.18
& 13.48$\pm$0.28
& 23.28$\pm$0.65
& 24.58$\pm$0.68
& 23.43$\pm$0.47
& 23.96$\pm$1.61
& \underline{12.81$\pm$0.52}
& 17.98$\pm$0.99
& 18.37$\pm$1.08
& \textbf{10.56$\pm$0.06} \\
& RMSE
& 33.99
& 20.38$\pm$0.32
& 20.23$\pm$0.33
& 34.02$\pm$0.77
& 34.37$\pm$0.57
& 33.90$\pm$0.39
& 33.66$\pm$1.64
& \underline{19.05$\pm$0.50}
& 26.86$\pm$0.97
& 26.71$\pm$1.78
& \textbf{16.77$\pm$0.08} \\

\midrule

\multirow{2}{*}{SX\_B}
& MAE
& 23.63
& 19.95$\pm$0.27
& \textbf{17.74$\pm$0.16}
& 28.34$\pm$0.28
& 28.56$\pm$0.39
& 27.66$\pm$0.71
& 28.30$\pm$1.40
& 19.17$\pm$0.71
& 22.88$\pm$0.89
& 22.88$\pm$2.10
& \underline{18.52$\pm$0.18} \\
& RMSE
& 28.38
& 25.07$\pm$0.33
& \textbf{22.88$\pm$0.32}
& 37.11$\pm$0.28
& 37.45$\pm$0.43
& 36.48$\pm$0.74
& 37.27$\pm$1.13
& 26.03$\pm$1.23
& 30.74$\pm$1.74
& 30.27$\pm$2.06
& \underline{23.83$\pm$0.44} \\

\multirow{2}{*}{SX\_C}
& MAE
& 13.00
& 8.07$\pm$0.26
& \underline{6.68$\pm$0.21}
& 10.45$\pm$0.18
& 10.71$\pm$0.45
& 10.43$\pm$0.61
& 10.75$\pm$1.04
& 8.77$\pm$0.58
& 8.59$\pm$0.59
& 7.81$\pm$1.05
& \textbf{6.17$\pm$0.11} \\
& RMSE
& 15.83
& 9.58$\pm$0.22
& \underline{8.46$\pm$0.19}
& 15.67$\pm$0.34
& 16.11$\pm$0.58
& 15.58$\pm$0.39
& 15.74$\pm$0.63
& 12.35$\pm$0.65
& 12.03$\pm$1.12
& 10.58$\pm$1.10
& \textbf{8.32$\pm$0.19} \\

\multirow{2}{*}{SX\_D}
& MAE
& 46.40
& 24.21$\pm$0.44
& 24.93$\pm$0.29
& 35.65$\pm$0.93
& 35.35$\pm$1.16
& 39.76$\pm$2.73
& 34.17$\pm$1.75
& \underline{19.13$\pm$0.55}
& 22.14$\pm$0.64
& 23.80$\pm$2.12
& \textbf{18.51$\pm$0.77} \\
& RMSE
& 57.77
& 29.33$\pm$0.60
& 29.84$\pm$0.53
& 47.47$\pm$1.48
& 46.24$\pm$0.67
& 52.55$\pm$2.71
& 44.91$\pm$2.26
& \textbf{24.51$\pm$0.60}
& 27.96$\pm$0.48
& 31.35$\pm$2.97
& \underline{24.70$\pm$0.67} \\

\multirow{2}{*}{SX\_E}
& MAE
& 4.41
& 2.99$\pm$0.08
& \underline{2.92$\pm$0.04}
& 5.70$\pm$0.19
& 5.44$\pm$0.10
& 5.58$\pm$0.05
& 5.77$\pm$0.11
& 3.30$\pm$0.18
& 4.23$\pm$0.22
& 4.03$\pm$0.19
& \textbf{2.90$\pm$0.02} \\
& RMSE
& 5.45
& 3.89$\pm$0.11
& \textbf{3.78$\pm$0.04}
& 7.62$\pm$0.24
& 7.07$\pm$0.29
& 7.25$\pm$0.04
& 7.36$\pm$0.16
& 4.52$\pm$0.26
& 5.51$\pm$0.33
& 5.43$\pm$0.31
& \underline{3.80$\pm$0.02} \\

\multirow{2}{*}{SX\_F}
& MAE
& 5.76
& 2.00$\pm$0.04
& \underline{1.95$\pm$0.04}
& 3.89$\pm$0.02
& 3.71$\pm$0.12
& 3.59$\pm$0.10
& 3.82$\pm$0.16
& \textbf{1.85$\pm$0.09}
& 2.96$\pm$0.23
& 2.50$\pm$0.17
& 2.23$\pm$0.09 \\
& RMSE
& 6.90
& \underline{2.82$\pm$0.04}
& \textbf{2.76$\pm$0.05}
& 6.01$\pm$0.02
& 5.99$\pm$0.05
& 5.90$\pm$0.05
& 5.92$\pm$0.13
& 3.01$\pm$0.09
& 4.57$\pm$0.31
& 4.01$\pm$0.27
& 3.24$\pm$0.06 \\

\multirow{2}{*}{SX\_G}
& MAE
& \underline{1.79}
& 1.97$\pm$0.05
& 2.13$\pm$0.04
& 3.07$\pm$0.04
& 2.65$\pm$0.10
& 2.80$\pm$0.09
& 2.63$\pm$0.19
& 2.10$\pm$0.08
& 2.50$\pm$0.08
& 2.56$\pm$0.05
& \textbf{1.51$\pm$0.09} \\
& RMSE
& \underline{2.32}
& 2.58$\pm$0.06
& 2.76$\pm$0.05
& 3.89$\pm$0.04
& 3.32$\pm$0.14
& 3.55$\pm$0.11
& 3.27$\pm$0.24
& 2.70$\pm$0.09
& 3.15$\pm$0.10
& 3.25$\pm$0.09
& \textbf{1.98$\pm$0.10} \\

\midrule

\multirow{2}{*}{AH\_B}
& MAE
& 18.54
& 9.82$\pm$0.10
& 9.99$\pm$0.15
& 13.98$\pm$0.20
& 13.88$\pm$0.32
& 13.32$\pm$0.73
& 13.83$\pm$0.15
& \underline{6.56$\pm$0.21}
& 8.44$\pm$0.47
& 10.10$\pm$0.45
& \textbf{5.64$\pm$0.12} \\
& RMSE
& 21.85
& 12.87$\pm$0.12
& 12.95$\pm$0.18
& 20.01$\pm$0.09
& 19.77$\pm$0.36
& 19.04$\pm$0.51
& 19.87$\pm$0.10
& \underline{10.34$\pm$0.49}
& 12.59$\pm$0.76
& 15.28$\pm$0.78
& \textbf{8.64$\pm$0.26} \\

\multirow{2}{*}{AH\_C}
& MAE
& 17.93
& 6.38$\pm$0.06
& 6.29$\pm$0.03
& 9.73$\pm$0.30
& 9.51$\pm$0.09
& 9.27$\pm$0.25
& 9.13$\pm$0.20
& \underline{3.78$\pm$0.11}
& 5.66$\pm$0.28
& 6.19$\pm$0.25
& \textbf{3.60$\pm$0.07} \\
& RMSE
& 19.80
& 8.72$\pm$0.14
& 8.54$\pm$0.08
& 13.51$\pm$0.23
& 13.22$\pm$0.34
& 13.32$\pm$0.19
& 13.31$\pm$0.32
& \underline{5.59$\pm$0.08}
& 7.95$\pm$0.33
& 9.51$\pm$0.17
& \textbf{5.07$\pm$0.11} \\

\multirow{2}{*}{AH\_D}
& MAE
& 22.30
& 9.32$\pm$0.11
& 9.35$\pm$0.13
& 12.93$\pm$0.48
& 14.05$\pm$0.94
& 13.04$\pm$0.56
& 13.13$\pm$0.27
& 8.88$\pm$0.65
& \underline{8.53$\pm$0.77}
& 9.15$\pm$0.27
& \textbf{7.58$\pm$0.10} \\
& RMSE
& 25.04
& 11.75$\pm$0.10
& 11.82$\pm$0.21
& 16.91$\pm$0.59
& 18.50$\pm$1.39
& 17.01$\pm$0.47
& 16.94$\pm$0.30
& 12.07$\pm$0.79
& \underline{11.43$\pm$0.86}
& 12.67$\pm$0.27
& \textbf{10.37$\pm$0.22} \\

\multirow{2}{*}{AH\_E}
& MAE
& 22.94
& 11.44$\pm$0.13
& 10.89$\pm$0.06
& 15.47$\pm$1.02
& 15.51$\pm$0.88
& 16.31$\pm$0.38
& 17.16$\pm$0.35
& \textbf{7.00$\pm$0.79}
& 9.67$\pm$0.45
& 10.16$\pm$0.35
& \underline{7.17$\pm$0.48} \\
& RMSE
& 26.78
& 13.96$\pm$0.10
& 13.24$\pm$0.01
& 21.04$\pm$1.40
& 20.46$\pm$1.10
& 21.00$\pm$0.38
& 21.85$\pm$0.56
& \underline{10.16$\pm$1.11}
& 12.69$\pm$0.52
& 14.08$\pm$0.63
& \textbf{9.77$\pm$0.52} \\

\multirow{2}{*}{AH\_F}
& MAE
& 43.77
& 17.83$\pm$0.31
& 17.88$\pm$0.10
& 25.72$\pm$0.25
& 25.49$\pm$0.23
& 25.09$\pm$0.13
& 25.95$\pm$0.11
& \textbf{11.46$\pm$0.29}
& 15.61$\pm$1.12
& 16.65$\pm$0.58
& \underline{11.52$\pm$0.26} \\
& RMSE
& 48.80
& 22.47$\pm$0.13
& 22.46$\pm$0.21
& 33.64$\pm$0.55
& 33.14$\pm$0.28
& 32.49$\pm$0.25
& 33.97$\pm$0.20
& \underline{16.56$\pm$0.54}
& 20.71$\pm$0.92
& 23.79$\pm$0.91
& \textbf{15.54$\pm$0.13} \\

\multirow{2}{*}{AH\_G}
& MAE
& 38.98
& 14.32$\pm$0.13
& 13.99$\pm$0.25
& 19.60$\pm$0.87
& 21.67$\pm$0.95
& 20.02$\pm$0.52
& 20.44$\pm$1.21
& \textbf{8.16$\pm$0.61}
& 12.90$\pm$0.67
& 14.80$\pm$0.37
& \underline{8.20$\pm$0.23} \\
& RMSE
& 43.59
& 18.22$\pm$0.19
& 17.69$\pm$0.30
& 26.45$\pm$0.72
& 28.23$\pm$0.95
& 26.57$\pm$0.57
& 27.14$\pm$1.10
& \underline{11.94$\pm$0.76}
& 17.84$\pm$0.84
& 20.58$\pm$0.63
& \textbf{11.15$\pm$0.11} \\

\multirow{2}{*}{AH\_H}
& MAE
& 18.34
& 8.28$\pm$0.06
& 8.30$\pm$0.10
& 12.35$\pm$0.23
& 12.92$\pm$0.57
& 12.56$\pm$0.34
& 12.30$\pm$0.28
& \underline{5.43$\pm$0.17}
& 7.74$\pm$0.54
& 8.41$\pm$0.72
& \textbf{4.89$\pm$0.07} \\
& RMSE
& 21.05
& 10.89$\pm$0.06
& 11.01$\pm$0.07
& 16.34$\pm$0.30
& 17.48$\pm$0.86
& 16.64$\pm$0.51
& 16.62$\pm$0.44
& \underline{8.18$\pm$0.29}
& 10.90$\pm$0.72
& 12.29$\pm$1.08
& \textbf{6.90$\pm$0.10} \\

\midrule

\multirow{2}{*}{Kelmarsh}
& MAE
& 2.82
& 1.47$\pm$0.01
& 1.48$\pm$0.02
& 3.35$\pm$0.09
& 3.65$\pm$0.18
& 3.46$\pm$0.06
& 3.62$\pm$0.23
& \underline{1.34$\pm$0.02}
& 1.80$\pm$0.09
& 1.98$\pm$0.24
& \textbf{1.33$\pm$0.02} \\
& RMSE
& 3.34
& \underline{1.83$\pm$0.01}
& 1.84$\pm$0.02
& 4.20$\pm$0.13
& 4.69$\pm$0.26
& 4.44$\pm$0.07
& 4.44$\pm$0.26
& \underline{1.83$\pm$0.04}
& 2.41$\pm$0.13
& 2.67$\pm$0.26
& \textbf{1.79$\pm$0.02} \\

\bottomrule
\end{tabular}%
}
\parbox{\linewidth}{
\scriptsize
\textit{Note: PowerCurve is deterministic. Therefore, no standard deviations are reported.}
}
\end{table*}

\begin{table*}[htbp]
\small
\centering
\renewcommand{\arraystretch}{1.3}
\caption{Average performance comparison of full-shot prediction across regions (\%).}
\label{tab:average_performance}
\resizebox{\textwidth}{!}{
\begin{tabular}{lcccccccc}
\toprule
\multirow{2}{*}{Method} &
\multicolumn{2}{c}{SD} &
\multicolumn{2}{c}{SX} &
\multicolumn{2}{c}{AH} &
\multicolumn{2}{c}{UK} \\
\cmidrule(lr){2-3}
\cmidrule(lr){4-5}
\cmidrule(lr){6-7}
\cmidrule(lr){8-9}
& NMAE & NRMSE & NMAE & NRMSE & NMAE & NRMSE & NMAE & NRMSE \\
\midrule

PowerCurve~\cite{haas_windpowerlib_2024}
& 23.02 & 31.13
& 24.03 & 29.85
& 39.14 & 44.61
& 28.99 & 35.02 \\
\hline

LightGBM~\cite{ju2019model}
& 12.56$\pm$0.07 & 18.40$\pm$0.09
& 15.02$\pm$0.18 & 19.12$\pm$0.17
& 17.06$\pm$0.07 & 21.90$\pm$0.07
& 17.85$\pm$0.08 & \underline{21.51$\pm$0.06} \\

XGBoost~\cite{phan2021hybrid}
& \underline{12.41$\pm$0.06} & \underline{18.21$\pm$0.08}
& 15.03$\pm$0.10 & \underline{19.09$\pm$0.12}
& 16.92$\pm$0.08 & 21.68$\pm$0.10
& 17.84$\pm$0.07 & 21.56$\pm$0.06 \\
\hline

PatchTST~\cite{nietime}
& 24.38$\pm$0.27 & 33.74$\pm$0.30
& 23.69$\pm$0.20 & 33.00$\pm$0.30
& 24.55$\pm$0.29 & 33.08$\pm$0.35
& 28.48$\pm$0.82 & 35.98$\pm$0.94 \\

iTransformer~\cite{liuitransformer}
& 24.97$\pm$0.26 & 33.53$\pm$0.35
& 22.78$\pm$0.28 & 31.83$\pm$0.38
& 25.22$\pm$0.38 & 33.75$\pm$0.51
& 29.56$\pm$0.81 & 37.82$\pm$1.09 \\

TimeMixer~\cite{wangtimemixer}
& 24.29$\pm$0.16 & 32.65$\pm$0.21
& 23.34$\pm$0.46 & 32.44$\pm$0.40
& 24.66$\pm$0.28 & 32.89$\pm$0.26
& 28.21$\pm$0.66 & 35.87$\pm$0.65 \\

xPatch~\cite{stitsyuk2025xpatch}
& 24.47$\pm$0.50 & 32.92$\pm$0.58
& 23.10$\pm$0.52 & 31.70$\pm$0.50
& 25.18$\pm$0.21 & 33.65$\pm$0.26
& 31.28$\pm$1.80 & 38.56$\pm$2.10 \\
\hline

CrossViViT~\cite{boussif2023improving}
& 12.55$\pm$0.15 & 18.73$\pm$0.24
& \underline{14.99$\pm$0.41} & 20.97$\pm$0.54
& \underline{11.16$\pm$0.26} & \underline{16.25$\pm$0.37}
& \underline{17.03$\pm$0.61} & 23.37$\pm$0.64 \\

FusionSF~\cite{ma2024fusionsf}
& 17.41$\pm$0.24 & 24.94$\pm$0.30
& 17.71$\pm$0.33 & 24.18$\pm$0.56
& 14.77$\pm$0.36 & 20.38$\pm$0.43
& 20.03$\pm$0.56 & 26.54$\pm$0.74 \\

2DXformer~\cite{zhang20252dxformer}
& 16.92$\pm$0.35 & 23.85$\pm$0.49
& 17.15$\pm$0.58 & 23.44$\pm$0.69
& 16.59$\pm$0.29 & 23.70$\pm$0.44
& 21.86$\pm$1.71 & 28.60$\pm$1.93 \\
\hline

\textbf{UniWind}
& \textbf{9.96$\pm$0.06} & \textbf{15.19$\pm$0.03}
& \textbf{12.89$\pm$0.20} & \textbf{17.42$\pm$0.19}
& \textbf{10.38$\pm$0.07} & \textbf{14.52$\pm$0.13}
& \textbf{15.93$\pm$0.05} & \textbf{20.87$\pm$0.06} \\

{Improvement (\%)}
& 19.74 & 16.58
& 14.01 & 8.75
& 6.99 & 10.65
& 6.48 & 3.02 \\
\bottomrule
\end{tabular}}
\parbox{\linewidth}{
\scriptsize
\textit{Note: PowerCurve is deterministic. Therefore, no standard deviations are reported.}
}
\end{table*}

\textbf{Cross-region zero-shot.}
This setting is stricter: for each target farm, a separate checkpoint is trained using data from regions other than the target region. UniWind remains best on all four datasets, with MAEs of 20.15, 13.88, 22.17, and 6.96 MW and RMSEs of 25.84, 19.81, 28.46, and 9.15 MW for SD\_A, SX\_A, AH\_A, and Penmanshiel, respectively. The improvements over the strongest foundation-model baseline are 15.44\%/27.60\% on SD\_A, 15.37\%/15.70\% on SX\_A, 26.15\%/31.57\% on AH\_A, and 8.42\%/10.82\% on Penmanshiel. Chronos-2 is the strongest baseline on SX\_A, while Moirai is strongest on the other three targets. Relative to cross-farm zero-shot, the performance margin becomes smaller, especially on SD\_A and AH\_A, which is expected because the model must transfer across both farm characteristics and regional distributions. Nevertheless, the consistent ranking across all targets indicates that UniWind's physical prior and state-aware correction remain useful under a stronger distribution shift.

\textbf{Generalization beyond full-shot training.}
Full-shot end-to-end models are trained separately on each target farm and therefore have access to target-specific supervision, whereas both zero-shot settings exclude target-farm labels from parameter fitting. UniWind's absolute errors increase under this restriction, but the degradation is moderate in the cross-farm setting and more pronounced in the cross-region setting. For example, its MAE changes from 11.96, 10.28, 12.26, and 6.04 MW in full-shot training to 14.43, 12.43, 15.05, and 6.56 MW in cross-farm zero-shot, and to 20.15, 13.88, 22.17, and 6.96 MW in cross-region zero-shot. More importantly, UniWind remains the best method on every target in both zero-shot settings and preserves substantial margins over foundation-model baselines. These results suggest that the proposed decomposition generalizes beyond target-specific end-to-end fitting: the shared physical prior transfers across farms and regions, while the state-aware correction adapts the prior using the information available at forecast time.

\subsection{Detailed Full-shot Performance}
\textbf{Full-shot performance on the remaining 20 wind farms.} 
Table~\ref{tab:additional_fullshot_transposed} extends the full-shot evaluation to the remaining 20 wind farms, including six SD farms, six SX farms, seven AH farms, and Kelmarsh. For each farm, every trainable model is trained on the corresponding training split and evaluated on its held-out test split. We report MAE and RMSE in MW as the mean $\pm$ standard deviation over five random seeds. This expanded evaluation tests whether the gains observed on the representative farms generalize across a broader range of site characteristics and power scales.

Overall, UniWind achieves the best or tied-best performance on the majority of dataset--metric pairs. The results are particularly consistent on the SD group: UniWind obtains the lowest MAE on all six farms and the lowest RMSE on all but SD\_D, where XGBoost is marginally better than UniWind (98.31 versus 98.43 MW). On the AH group, UniWind achieves the best RMSE on all seven farms. CrossViViT attains slightly lower MAE on AH\_E, AH\_F, and AH\_G, but UniWind remains highly competitive on these farms and preserves a clear advantage in RMSE, indicating that its forecasts are less affected by larger errors.

The SX group exhibits greater variation across sites, which provides a useful test of model robustness. UniWind is best on both metrics for SX\_C and SX\_G, and achieves the best MAE on SX\_D and SX\_E. The main exceptions occur on SX\_B and SX\_F, where XGBoost or CrossViViT performs better, while CrossViViT and XGBoost also obtain slightly lower RMSE on SX\_D and SX\_E, respectively. These cases suggest that site-specific error distributions can favor different inductive biases, but they do not change the overall advantage of UniWind across the expanded evaluation.

Across baseline families, the historical-sequence models PatchTST, iTransformer, TimeMixer, and xPatch generally lag behind models that exploit future NWP information or explicit weather--power structure. LightGBM and XGBoost remain strong on several individual farms, especially within the SX group and on the SD\_D RMSE metric, whereas CrossViViT is the most competitive renewable-energy baseline and occasionally achieves the lowest MAE. The physics-only PowerCurve is generally less competitive than the learned models, showing that a fixed power curve alone cannot accommodate site-specific responses and latent operating effects. UniWind's consistent advantage over these complementary baselines supports the benefit of combining a site-calibrated physical prior with state-aware correction rather than relying on either a fixed physical curve or unconstrained sequence modeling alone. The standard deviations of UniWind are also generally small or comparable to those of the strongest baselines, indicating stable performance across random seeds. Taken together, these results show that the proposed design remains effective beyond the four representative farms and across heterogeneous full-shot forecasting scenarios.

\textbf{Aggregate full-shot performance.} Table~\ref{tab:average_performance} reports the average full-shot performance across four regions. Since the rated power varies across wind farms, we report normalized metrics, defined as NMAE = MAE/rated power $\times$ 100\% and NRMSE = RMSE/rated power $\times$ 100\%, to ensure fair comparisons. UniWind consistently achieves the lowest NMAE and NRMSE across all four regions, outperforming both time-series forecasting models and renewable-energy forecasting baselines. Specifically, compared with the strongest competing baseline in each region, UniWind reduces NMAE by 19.74\%, 14.01\%, 6.99\%, and 6.48\% in SD, SX, AH, and UK, respectively, while achieving corresponding NRMSE reductions of 16.58\%, 8.75\%, 10.65\%, and 3.02\%. Notably, the strongest competing methods vary across regions and evaluation metrics, indicating that existing approaches exhibit region-dependent performance. In contrast, UniWind maintains consistent improvements across all settings, suggesting its adaptability to heterogeneous regional wind-power distributions. Furthermore, the relatively small standard deviations indicate that UniWind delivers stable prediction performance across repeated runs.

\begin{table}[!htbp]
\scriptsize
\centering
\caption{Additional cross-region zero-shot performance on the remaining 20 wind farms.}
\label{tab:additional_cross-region_transposed}
\renewcommand{\arraystretch}{1.3}
\setlength{\tabcolsep}{3pt}
\resizebox{0.48\textwidth}{!}{
\begin{tabular}{llcccc}
\toprule
Dataset & Metric
& WindFM
& Moirai
& Chronos-2
& \textbf{UniWind} \\
\midrule

\multirow{2}{*}{SD\_B} % SD_W_B
& MAE  & $24.91$ & $\mathbf{15.62{\pm}0.28}$ & $18.19{\pm}0.95$ & $\underline{15.75{\pm}0.25}$ \\
& RMSE & $34.49$ & $\underline{24.23{\pm}0.22}$ & $26.57{\pm}1.56$ & $\mathbf{22.21{\pm}0.63}$ \\

\multirow{2}{*}{SD\_C} % SD_W_C
& MAE  & $54.16$ & $\underline{35.26{\pm}0.39}$ & $37.71{\pm}2.80$ & $\mathbf{33.75{\pm}2.22}$ \\
& RMSE & $73.79$ & $\underline{51.48{\pm}1.07}$ & $53.61{\pm}4.08$ & $\mathbf{48.16{\pm}1.67}$ \\

\multirow{2}{*}{SD\_D} % SD_W_D
& MAE  & $149.66$ & $124.43{\pm}0.83$ & $\underline{121.04{\pm}7.91}$ & $\mathbf{94.40{\pm}4.09}$ \\
& RMSE & $201.63$ & $\underline{168.83{\pm}2.16}$ & $177.02{\pm}12.25$ & $\mathbf{128.53{\pm}3.57}$ \\

\multirow{2}{*}{SD\_E} % SD_W_J
& MAE  & $44.09$ & $\mathbf{31.33{\pm}0.81}$ & $33.33{\pm}2.89$ & $\underline{31.55{\pm}1.83}$ \\
& RMSE & $61.75$ & $\underline{46.11{\pm}0.30}$ & $48.43{\pm}2.56$ & $\mathbf{42.01{\pm}1.24}$ \\

\multirow{2}{*}{SD\_F} % SD_W_K
& MAE  & $49.76$ & $\underline{28.69{\pm}0.24}$ & $29.56{\pm}3.72$ & $\mathbf{24.40{\pm}1.45}$ \\
& RMSE & $65.33$ & $\underline{41.28{\pm}0.56}$ & $43.03{\pm}3.64$ & $\mathbf{32.84{\pm}1.14}$ \\

\multirow{2}{*}{SD\_G} % SD_W_M
& MAE  & $25.14$ & $17.09{\pm}0.34$ & $\mathbf{16.23{\pm}1.60}$ & $\underline{16.66{\pm}1.79}$ \\
& RMSE & $36.14$ & $26.06{\pm}0.22$ & $\underline{24.35{\pm}1.57}$ & $\mathbf{22.20{\pm}1.59}$ \\

\midrule

\multirow{2}{*}{SX\_B} % SX_W_67
& MAE  & $28.67$ & $\underline{25.99{\pm}0.21}$ & $36.64{\pm}3.57$ & $\mathbf{22.45{\pm}0.69}$ \\
& RMSE & $39.64$ & $\underline{34.26{\pm}0.59}$ & $59.25{\pm}10.21$ & $\mathbf{29.63{\pm}1.02}$ \\

\multirow{2}{*}{SX\_C} % SX_W_211
& MAE  & $11.00$ & $11.52{\pm}0.12$ & $\underline{10.65{\pm}0.45}$ & $\mathbf{6.64{\pm}0.14}$ \\
& RMSE & $16.31$ & $15.30{\pm}0.22$ & $\underline{15.27{\pm}1.13}$ & $\mathbf{8.78{\pm}0.21}$ \\

\multirow{2}{*}{SX\_D} % SX_W_223
& MAE  & $44.90$ & $\mathbf{28.25{\pm}0.31}$ & $43.53{\pm}2.48$ & $\underline{29.12{\pm}0.96}$ \\
& RMSE & $55.37$ & $\underline{39.25{\pm}0.52}$ & $56.06{\pm}2.94$ & $\mathbf{37.28{\pm}1.28}$ \\

\multirow{2}{*}{SX\_E} % SX_W_261
& MAE  & $5.34$ & $\underline{3.84{\pm}0.03}$ & $7.05{\pm}0.79$ & $\mathbf{3.42{\pm}0.08}$ \\
& RMSE & $7.17$ & $\underline{5.44{\pm}0.10}$ & $12.92{\pm}2.04$ & $\mathbf{4.41{\pm}0.18}$ \\

\multirow{2}{*}{SX\_F} % SX_W_272
& MAE  & $3.56$ & $4.22{\pm}0.02$ & $\underline{3.18{\pm}0.13}$ & $\mathbf{2.74{\pm}0.22}$ \\
& RMSE & $6.21$ & $5.89{\pm}0.09$ & $\underline{4.54{\pm}0.22}$ & $\mathbf{3.54{\pm}0.21}$ \\

\multirow{2}{*}{SX\_G} % SX_W_295
& MAE  & $2.78$ & $\mathbf{1.51{\pm}0.03}$ & $2.63{\pm}0.25$ & $\underline{1.63{\pm}0.03}$ \\
& RMSE & $3.46$ & $\underline{2.34{\pm}0.04}$ & $3.93{\pm}0.44$ & $\mathbf{2.05{\pm}0.03}$ \\

\midrule

\multirow{2}{*}{AH\_B} % AH_W_B
& MAE  & $13.71$ & $\underline{9.72{\pm}0.05}$ & $11.20{\pm}0.44$ & $\mathbf{7.55{\pm}0.29}$ \\
& RMSE & $20.29$ & $\underline{13.45{\pm}0.10}$ & $16.62{\pm}0.63$ & $\mathbf{11.02{\pm}0.17}$ \\

\multirow{2}{*}{AH\_C} % AH_W_C
& MAE  & $10.01$ & $10.01{\pm}0.13$ & $\underline{8.32{\pm}0.16}$ & $\mathbf{7.43{\pm}0.10}$ \\
& RMSE & $13.87$ & $13.88{\pm}0.12$ & $\underline{11.51{\pm}0.13}$ & $\mathbf{9.88{\pm}0.07}$ \\

\multirow{2}{*}{AH\_D} % AH_W_D
& MAE  & $13.57$ & $\mathbf{10.33{\pm}0.11}$ & $12.20{\pm}0.36$ & $\underline{11.17{\pm}0.16}$ \\
& RMSE & $17.39$ & $\mathbf{14.20{\pm}0.13}$ & $16.35{\pm}0.59$ & $\underline{14.48{\pm}0.27}$ \\

\multirow{2}{*}{AH\_E} % AH_W_E
& MAE  & $16.87$ & $13.28{\pm}0.13$ & $\underline{13.20{\pm}0.16}$ & $\mathbf{11.30{\pm}0.07}$ \\
& RMSE & $21.54$ & $18.30{\pm}0.17$ & $\underline{17.15{\pm}0.27}$ & $\mathbf{14.62{\pm}0.11}$ \\

\multirow{2}{*}{AH\_F} % AH_W_F
& MAE  & $25.50$ & $23.09{\pm}0.33$ & $\underline{19.86{\pm}0.23}$ & $\mathbf{16.89{\pm}0.15}$ \\
& RMSE & $33.45$ & $30.44{\pm}0.32$ & $\underline{26.36{\pm}0.41}$ & $\mathbf{22.30{\pm}0.25}$ \\

\multirow{2}{*}{AH\_G} % AH_W_G
& MAE  & $23.62$ & $19.32{\pm}0.26$ & $\underline{17.37{\pm}0.29}$ & $\mathbf{15.31{\pm}0.13}$ \\
& RMSE & $29.69$ & $26.63{\pm}0.24$ & $\underline{21.87{\pm}0.36}$ & $\mathbf{19.57{\pm}0.16}$ \\

\multirow{2}{*}{AH\_H} % AH_W_H
& MAE  & $11.61$ & $10.93{\pm}0.18$ & $\underline{9.74{\pm}0.14}$ & $\mathbf{8.48{\pm}0.11}$ \\
& RMSE & $15.42$ & $15.03{\pm}0.13$ & $\underline{13.60{\pm}0.07}$ & $\mathbf{11.28{\pm}0.10}$ \\

\midrule

\multirow{2}{*}{Kelmarsh} % UK_Kelmarsh
& MAE  & $3.85$ & $\underline{2.62{\pm}0.05}$ & $2.91{\pm}0.07$ & $\mathbf{2.07{\pm}0.05}$ \\
& RMSE & $4.89$ & $\underline{3.63{\pm}0.07}$ & $3.96{\pm}0.08$ & $\mathbf{2.77{\pm}0.07}$ \\
\bottomrule
\end{tabular}}
\parbox{\linewidth}{%
\scriptsize
\textit{Note: 
WindFM is evaluated using a released checkpoint. Therefore, no standard deviations are reported.}
}
\end{table}

\subsection{Detailed Cross-region Zero-shot Performance}
Table~\ref{tab:additional_cross-region_transposed} reports cross-region zero-shot performance on the remaining 20 wind farms. UniWind achieves the lowest MAE on 14 farms and the lowest RMSE on 19, ranking first on 33 of the 40 dataset--metric pairs and second on the remaining seven. These results extend the performance advantage observed on the four representative farms to a broader set of targets.

UniWind achieves the lowest RMSE on all twelve SD and SX farms. On SD\_D, it reduces MAE and RMSE by 22.01\% and 23.87\%, respectively, relative to the strongest baseline for each metric. On SX\_C, the corresponding reductions reach 37.65\% and 42.50\% relative to Chronos-2. The MAE results are less uniform: Moirai performs better on SD\_B, SD\_E, SX\_D, and SX\_G, while Chronos-2 achieves the lowest MAE on SD\_G. UniWind nevertheless retains the lowest RMSE on these targets.

On the AH farms, UniWind achieves the lowest MAE and RMSE on six of the seven targets, with AH\_D being the only exception, where Moirai leads on both metrics. On Kelmarsh, UniWind reduces MAE and RMSE by 20.99\% and 23.69\%, respectively, relative to Moirai. Overall, the expanded evaluation supports the cross-region applicability of UniWind, while showing that its advantage varies across sites and is more consistent for RMSE than for MAE.

\begin{figure*}[!htp]
    \centering
    \subfigure[SD\_A]{
    \includegraphics[width=0.232\textwidth]{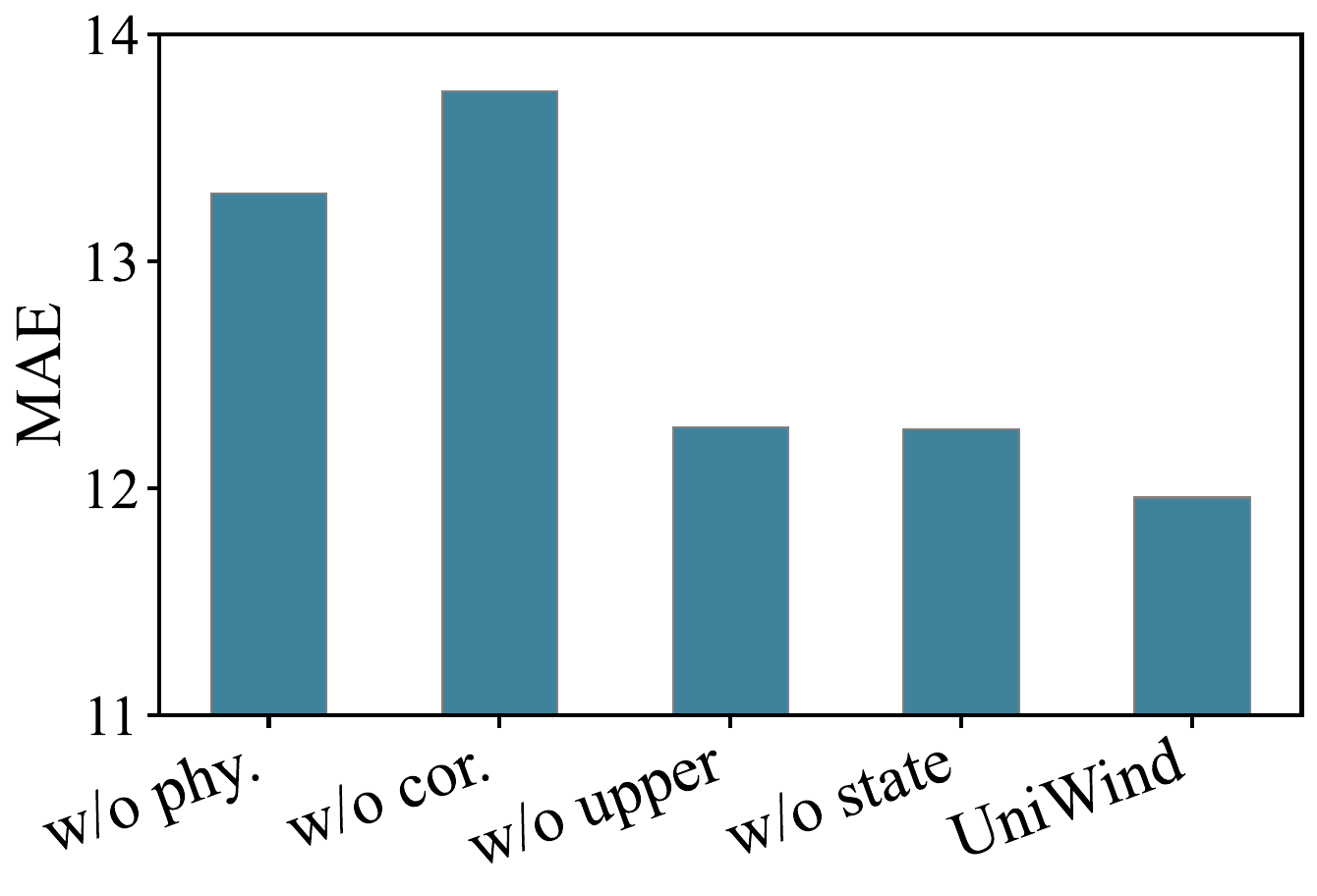}
    }
    \subfigure[SX\_A]{
    \includegraphics[width=0.232\textwidth]{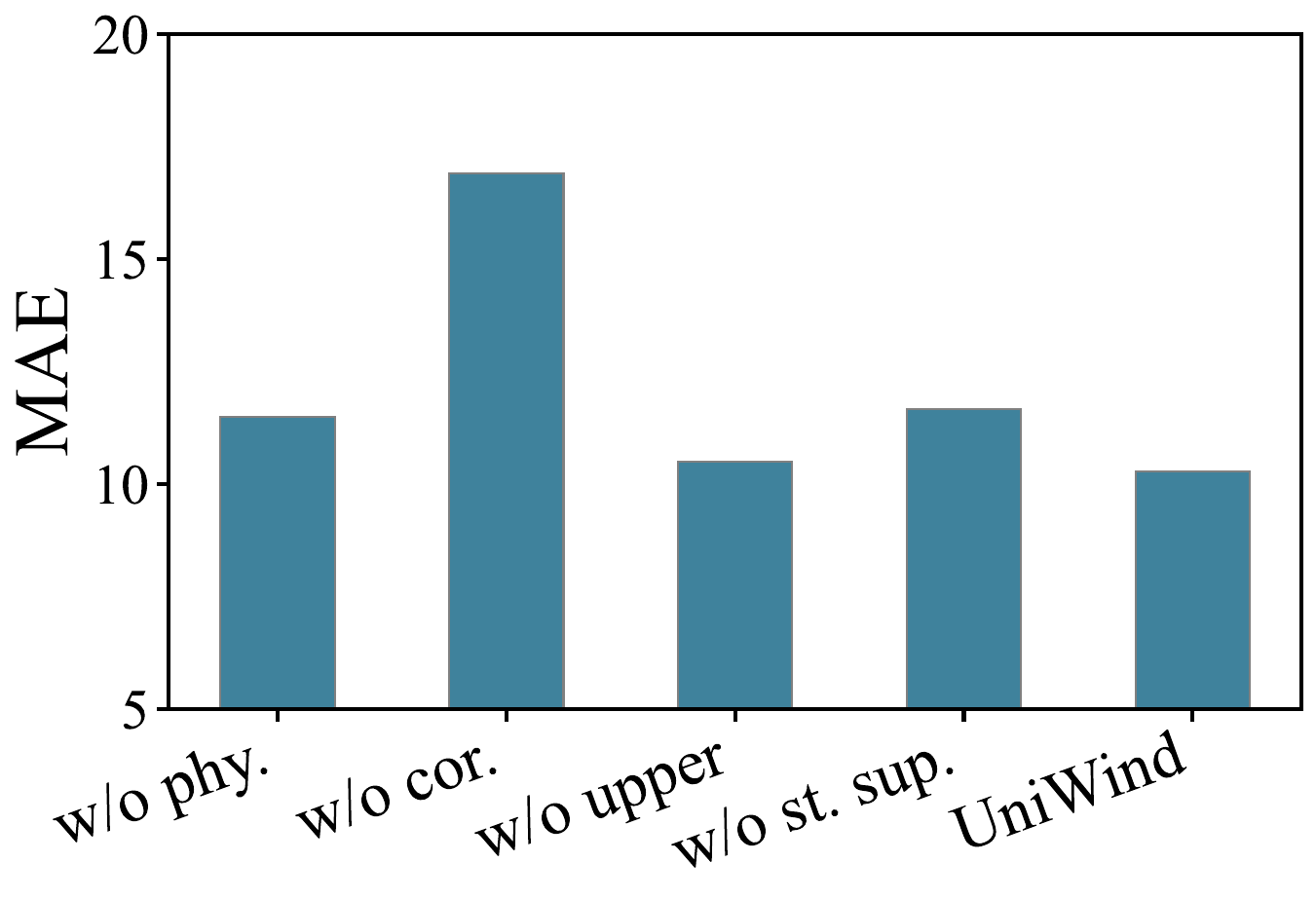}
    }
    \subfigure[AH\_A]{
    \includegraphics[width=0.232\textwidth]{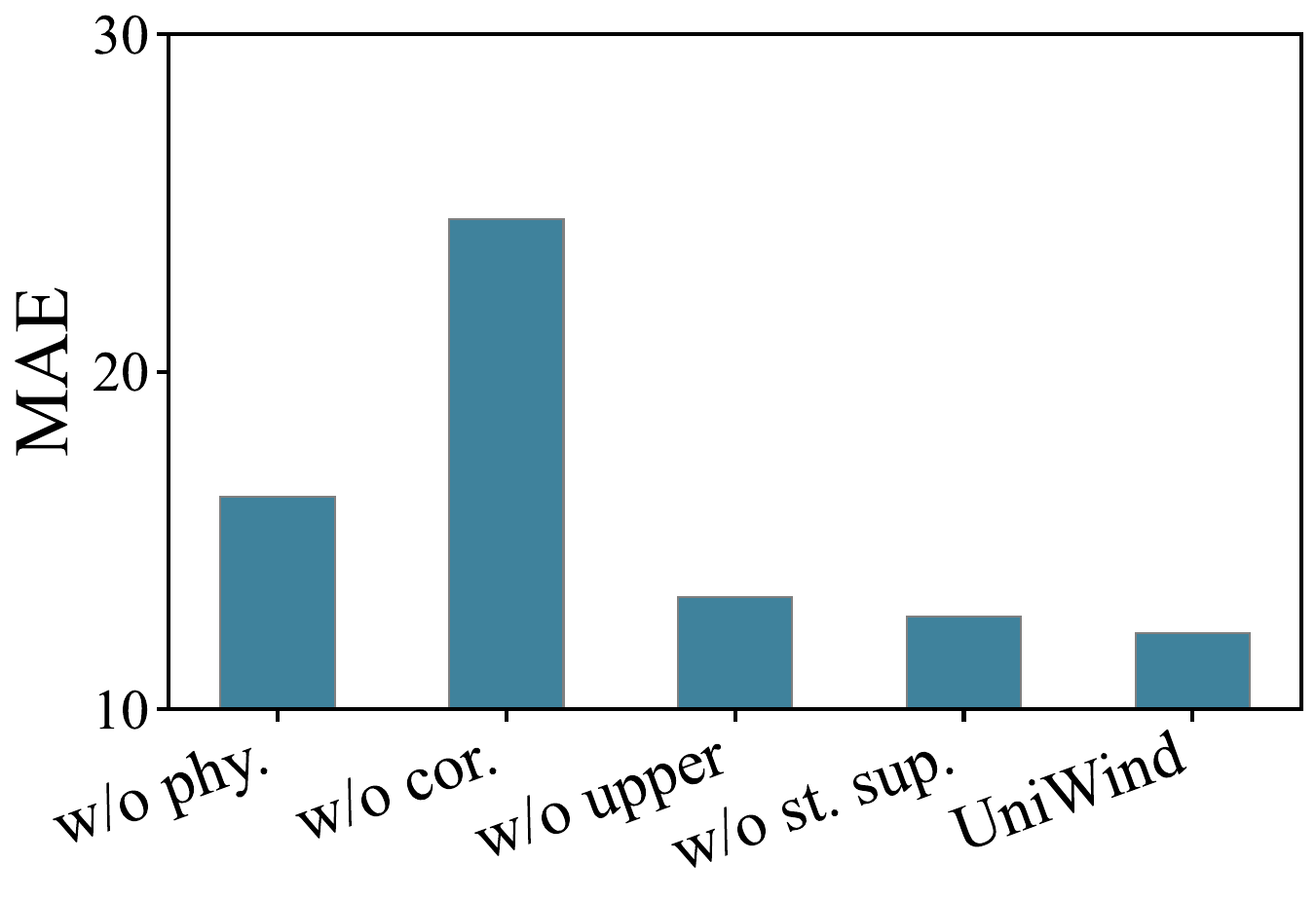}
    }
    \subfigure[Penmanshiel]{
    \includegraphics[width=0.232\textwidth]{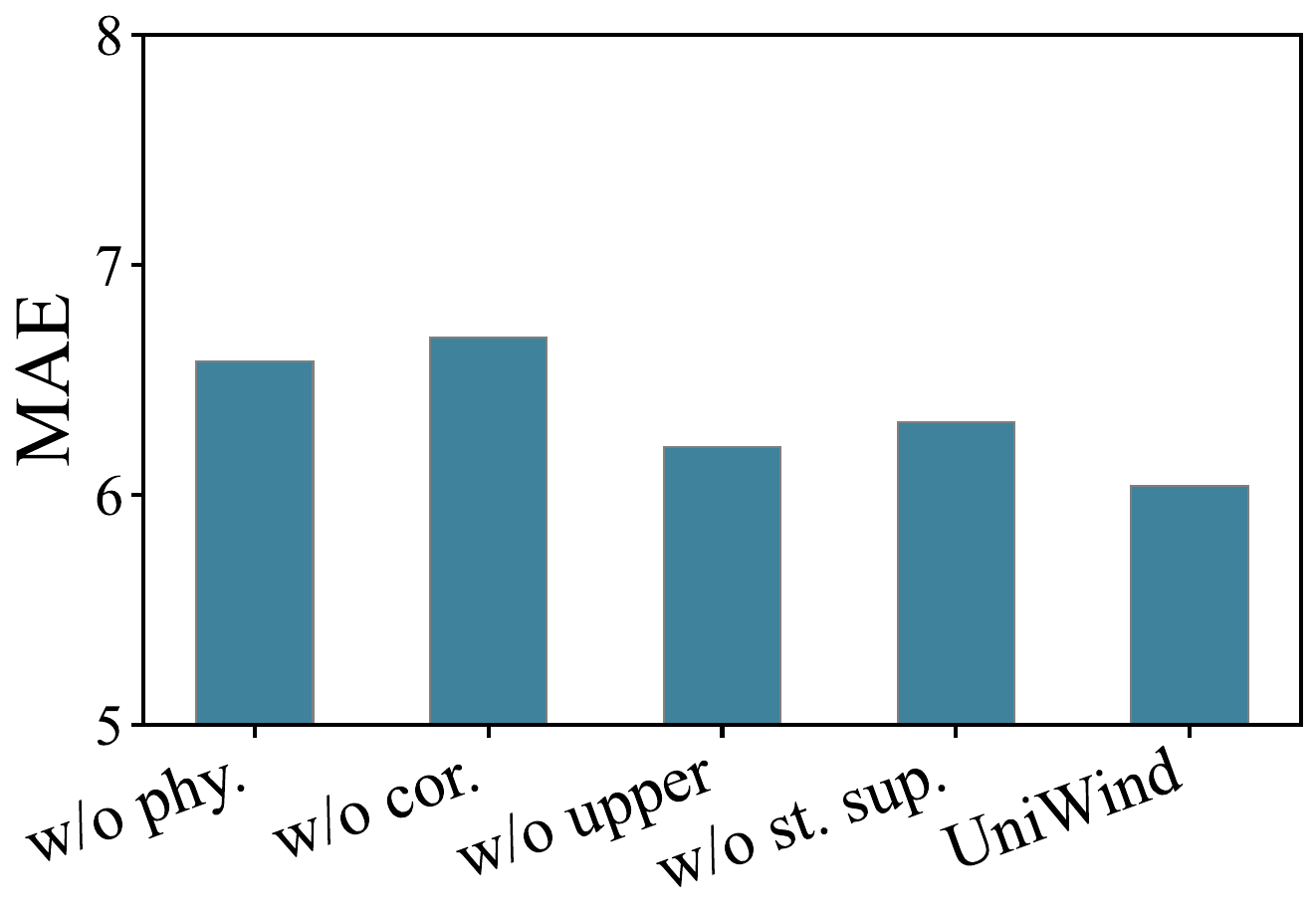}
    }
    \caption{Ablation Studies.}
    \label{fig:all_Ablation}
\end{figure*}
\subsection{Ablation Studies}
\subsubsection{Main Component Ablation}

To isolate the contribution of the major design choices in UniWind, we conduct a component-level ablation study on four representative datasets: SD\_A, SX\_A, AH\_A, and Penmanshiel. Each variant is evaluated under the same protocol as the full model, and the reported results are averaged over five random seeds. Specifically, we remove one component at a time while keeping the remaining architecture, training data, and optimization settings unchanged. The evaluated variants are defined as follows:
\begin{itemize}
    \item \textbf{w/o phy.} removes the physical prior and replaces it with a learnable parameter, thereby eliminating the physically structured reference used by the correction module.
    \item \textbf{w/o cor.} removes the State-aware Power Corrector and directly uses the physical prior as the final prediction.
    \item \textbf{w/o upper} removes the physical upper-bound constraint $\mathcal{L}_{\mathrm{prior}}$ from the Physical Prior Estimator while retaining the other components of prior estimation.
    \item \textbf{w/o st. sup.} removes the pseudo-label-supervised state routing loss $\mathcal{L}_{\mathrm{router}}$ while retaining the routing and state-specific correction modules.
\end{itemize}

Figure~\ref{fig:all_Ablation} shows that the complete UniWind consistently outperforms all ablated variants across the four representative datasets. The results lead to four observations. First, the degradation of every variant confirms that the physical prior, bounded correction, upper-bound regularization, and state supervision provide complementary benefits. Second, \textbf{w/o cor.} exhibits the largest performance drop, showing that the physical prior alone cannot describe realized wind power under latent and time-varying operating states; the prior should be transformed by the State-aware Power Corrector. Third, \textbf{w/o phy.} also performs substantially worse than the full model, indicating that a physically motivated prior provides a useful reference for separating meteorology-driven power potential from state-dependent deviations. Finally, removing either the upper-bound constraint or the routing supervision causes a smaller but still consistent degradation. This result suggests that the upper-bound constraint helps shape the prior into a meaningful soft envelope, while $\mathcal{L}_{\mathrm{router}}$ improves the alignment between inferred states and the corresponding state-specific experts. Together, these results support the proposed decomposition of wind-power forecasting into physical-prior estimation and state-aware correction.

\subsubsection{Cross-farm Warp Ablation}
\label{sec:ablation_cross_farm_warp}

The site-conditioned monotone warp is designed for multi-farm training and mainly affects cross-farm generalization. We therefore ablate this component in the cross-farm zero-shot setting, where the model is trained on source farms and evaluated on previously unseen target farms. Specifically, we remove the site-conditioned monotone warp from the Physical Prior Estimator while keeping the other components unchanged.

Table~\ref{tab:ablation_cross_farm_warp} summarizes the performance of the resulting variant on four target wind farms. Removing the site-conditioned monotone warp consistently degrades zero-shot forecasting performance, indicating that site-conditioned monotonic calibration is important for adapting the shared physical prior to farm-specific characteristics and improving cross-farm transferability.

\begin{table}[!htbp]
\small
\centering
\caption{Comparison between UniWind and its variant without the site-conditioned monotone warp (w/o smw for short) under cross-farm zero-shot forecasting.}
\label{tab:ablation_cross_farm_warp}
\resizebox{0.48\textwidth}{!}{%
\begin{tabular}{lcccc}
\toprule
& \multicolumn{2}{c}{UniWind}
& \multicolumn{2}{c}{w/o smw} \\
\cmidrule(lr){2-3}
\cmidrule(lr){4-5}
Dataset & MAE & RMSE & MAE & RMSE \\
\midrule

SD\_A
& \textbf{14.43$\pm$0.24}
& \textbf{22.26$\pm$0.27}
& 16.15$\pm$0.36
& 24.75$\pm$0.79 \\

SX\_A
& \textbf{12.43$\pm$0.13}
& \textbf{18.71$\pm$0.19}
& 12.73$\pm$0.16
& 18.96$\pm$0.23 \\

AH\_A
& \textbf{15.05$\pm$0.16}
& \textbf{19.12$\pm$0.43}
& 15.80$\pm$0.46
& 19.97$\pm$0.75 \\
Penmanshiel
& \textbf{6.56$\pm$0.09}
& \textbf{8.71$\pm$0.09}
& 6.91$\pm$0.07
& 9.10$\pm$0.23 \\
\bottomrule
\end{tabular}}
\end{table}

\begin{figure}[htbp]
    \centering
    \subfigure[$\lambda_{\mathrm{router}}$]{
    \includegraphics[width=0.225\textwidth]{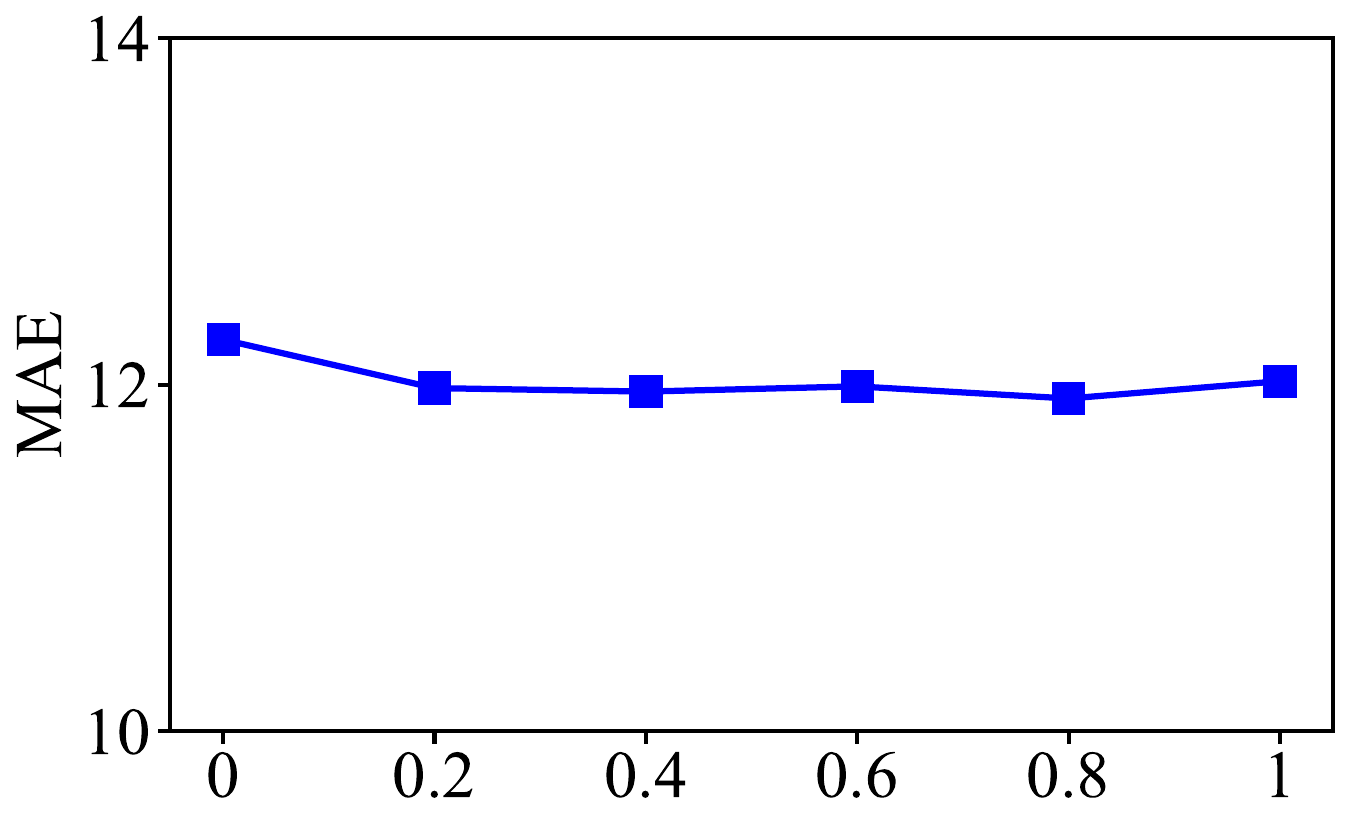}
    }
    \subfigure[$\lambda_{\mathrm{upper}}$]{
    \includegraphics[width=0.225\textwidth]{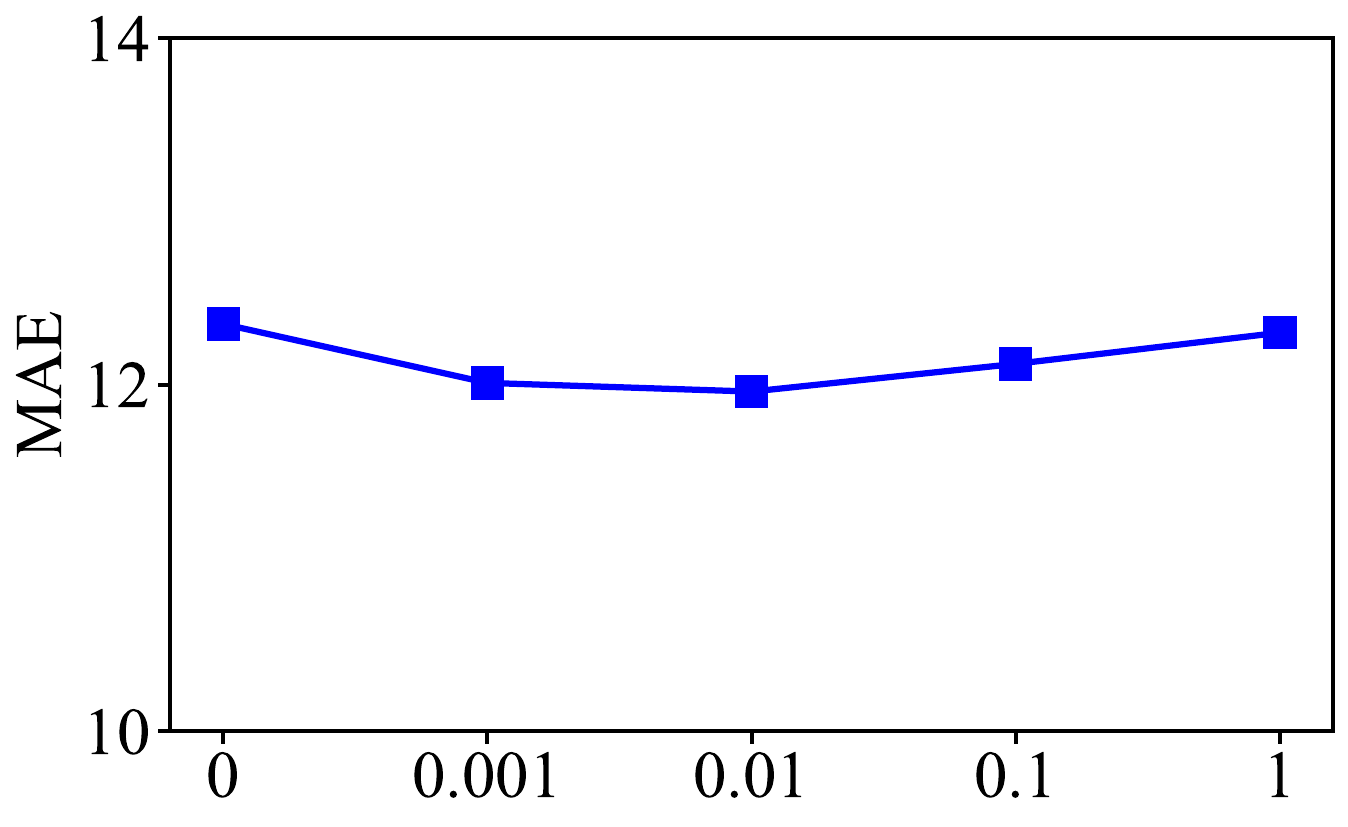}
    }
    \caption{Sensitivity analysis on the SD\_A dataset.}
    \label{fig:Sensitivity}
\end{figure}

\begin{figure*}[!htp]
    \centering
    \subfigure[Low-wind, regular]{
    \includegraphics[width=0.232\textwidth]{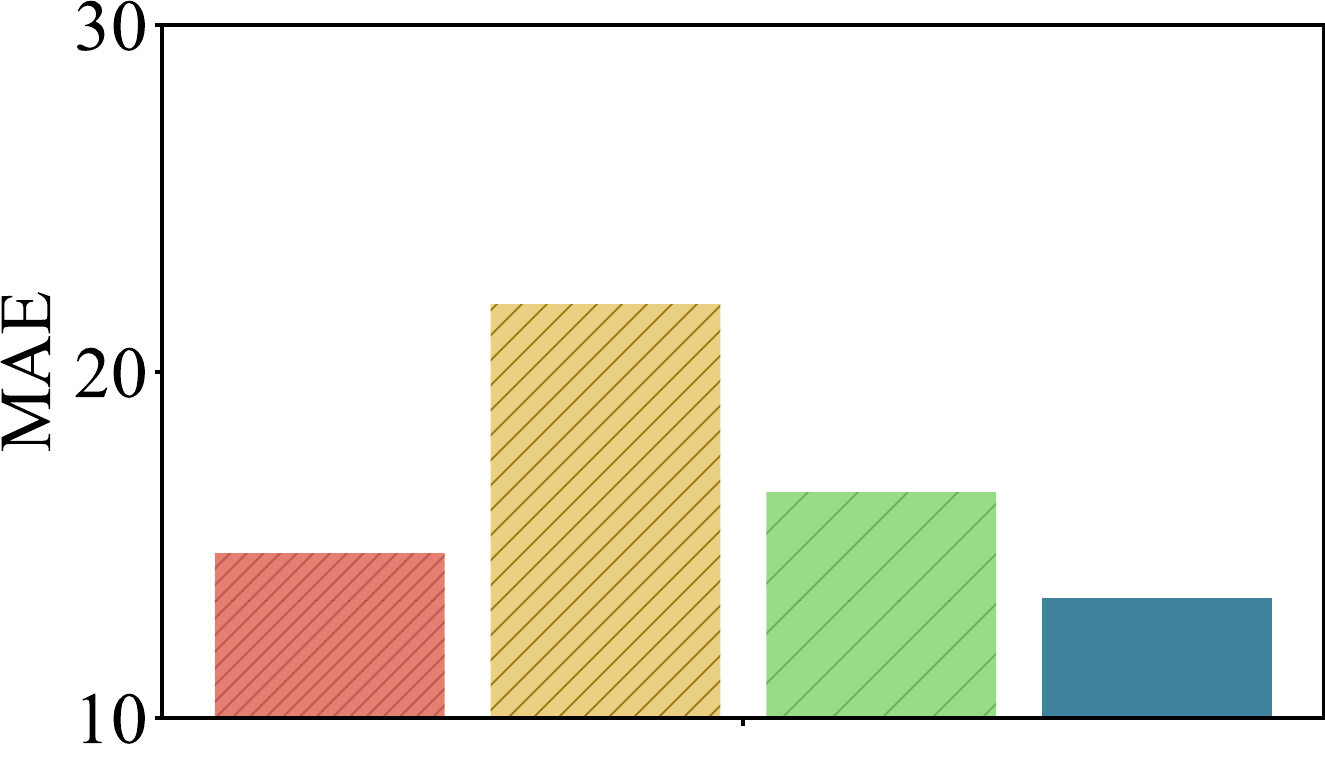}
    }
    \subfigure[Low-wind, abnormal]{
    \includegraphics[width=0.232\textwidth]{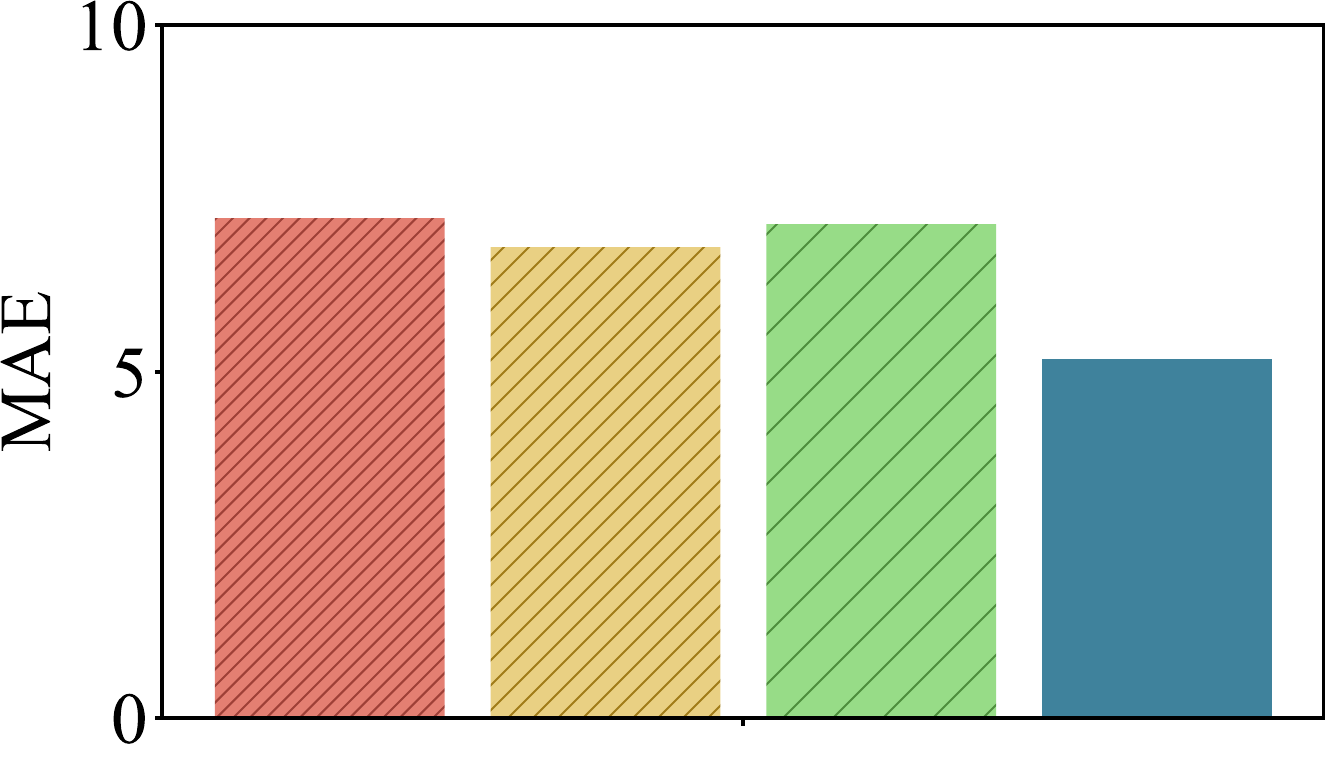}
    }
    \subfigure[High-wind, regular]{
    \includegraphics[width=0.232\textwidth]{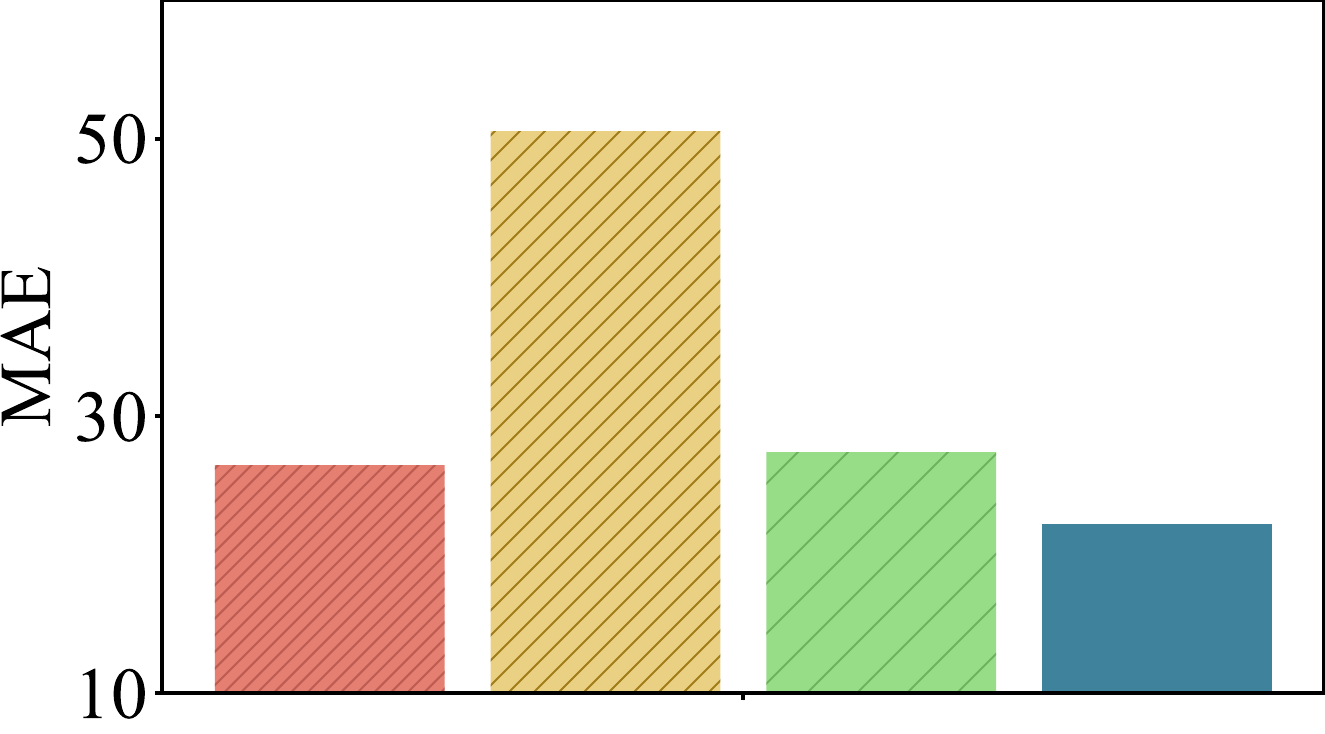}
    }
    \subfigure[High-wind, abnormal]{
    \includegraphics[width=0.232\textwidth]{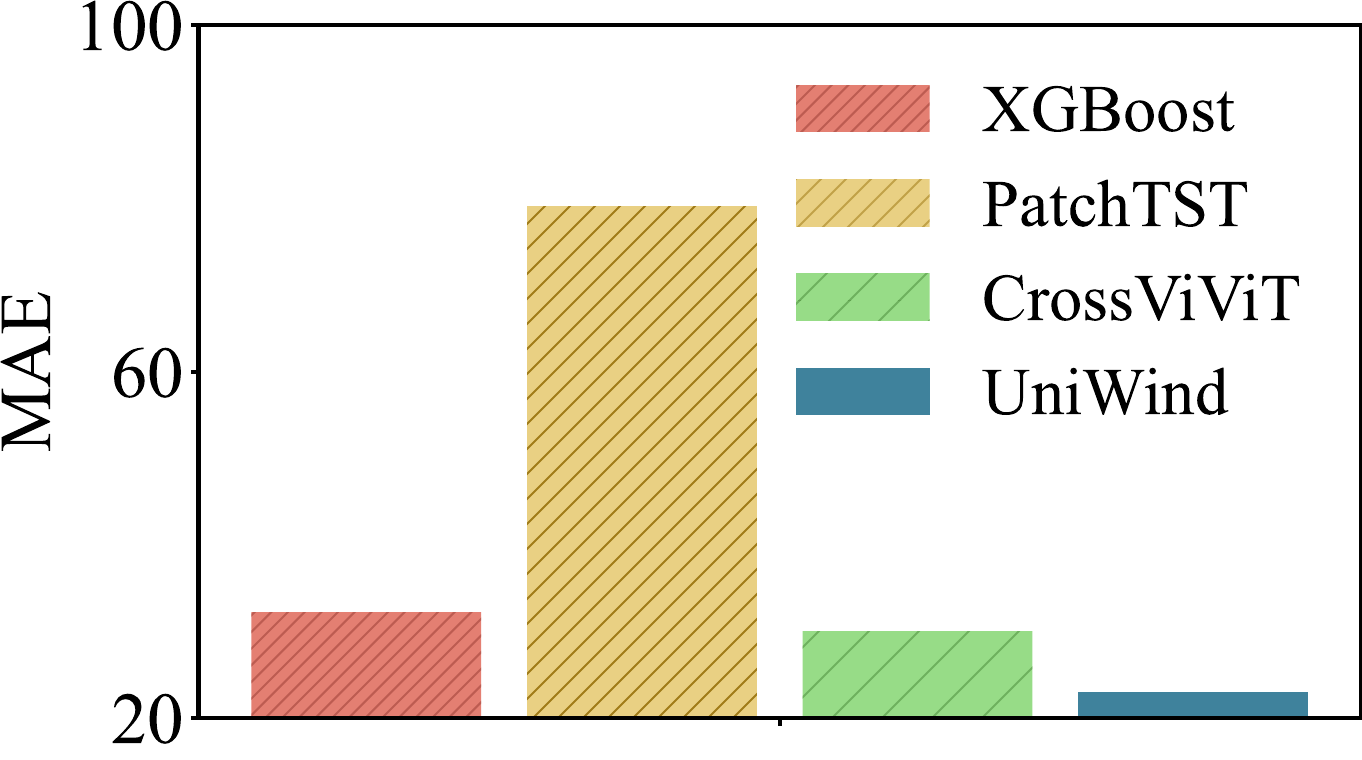}
    }
    \caption{Performance comparison under varied conditions on the SD\_A dataset.}
    \label{fig:condition_performance}
\end{figure*}
\subsection{Sensitivity Analysis}

Figure~\ref{fig:Sensitivity} shows how UniWind's performance varies with the router-supervision weight \(\lambda_{\mathrm{router}}\) and the prior-tightness weight \(\lambda_{\mathrm{upper}}\). The analysis is conducted on SD\_A by varying one weight while keeping the architecture, data split, training protocol, and all other hyperparameters fixed. The router-supervision weight \(\lambda_{\mathrm{router}}\) controls the contribution of the pseudo-label-supervised routing loss \(\mathcal{L}_{\mathrm{router}}\), whereas \(\lambda_{\mathrm{upper}}\) controls the tightness regularization that encourages the physical prior to remain a useful upper envelope of available generation.

When \(\lambda_{\mathrm{router}}=0\), the routing module no longer receives direct pseudo-label guidance and must learn expert assignment only through the final forecasting objective. The resulting increase in MAE indicates that the forecasting loss alone is insufficient to reliably separate the latent operating states or to specialize the corresponding experts. This result is consistent with the role of \(\mathcal{L}_{\mathrm{router}}\) as weak supervision: it does not replace the forecasting objective, but provides an additional training signal that helps align the inferred state representation with the intended correction behavior. Once \(\lambda_{\mathrm{router}}\) is nonzero, the performance varies within a relatively narrow range, suggesting that the routing mechanism benefits primarily from having a meaningful supervision signal rather than from a highly precise weight value.

A similar pattern is observed for \(\lambda_{\mathrm{upper}}\). With \(\lambda_{\mathrm{upper}}=0\), the regular-state tightness term is removed, while the coverage term in \(\mathcal{L}_{\mathrm{prior}}\) remains active. The increase in MAE supports the contribution of the tightness term to forecasting performance.
For nonzero values, the performance remains comparatively stable, indicating that the physical prior does not require an overly strong regularization coefficient to provide a useful envelope. The regularizer therefore acts as a structural constraint that stabilizes prior estimation without dominating the data-driven correction objective.

Taken together, the two curves show that both weights are important at the level of mechanism, while UniWind is not excessively sensitive to their exact nonzero values. This provides a practical balance between knowledge guidance and data-driven learning: removing either signal causes a clear degradation, but moderate changes around the selected values do not lead to a substantial performance collapse. Based on validation performance, we set \(\lambda_{\mathrm{router}}=0.4\) and \(\lambda_{\mathrm{upper}}=0.01\) in the final model. These are the default weights for the complete model; ablation and sensitivity experiments modify the relevant terms as specified.

\subsection{State-aware Analysis}

\begin{figure*}[!htp]
    \centering
    \includegraphics[width=0.6\linewidth]{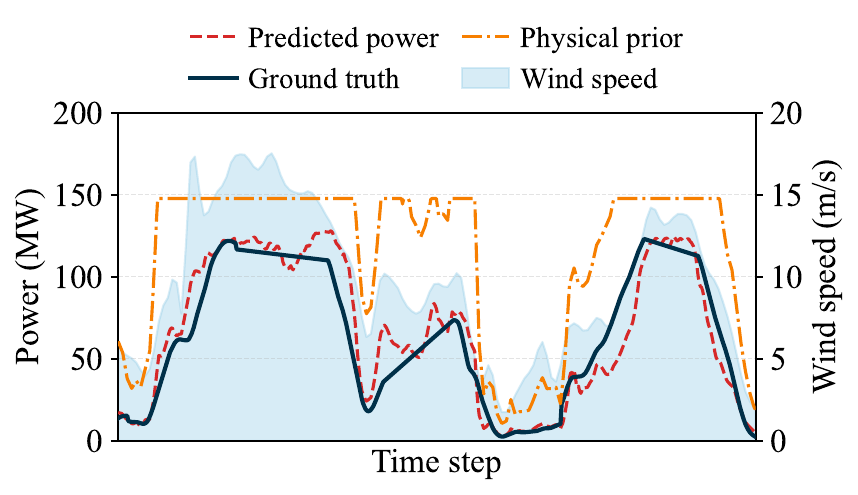}
    \caption{Case study on the SD\_A dataset. The physical prior follows the wind-speed-driven available power and saturates near the rated power, while UniWind corrects this prior into the final realized-power forecast.}
    \label{fig:case_study}
\end{figure*}

\subsubsection{State Distribution and Interpretation}
Table~\ref{tab:wind_power_States} summarizes the average proportions of rule-derived state pseudo-labels constructed by the procedure in Section~\ref{sec:state_prior_labels}. The regular state dominates all regions, but the proportion varies substantially, from 60.5\% in the UK to 84.3\% in AH. The abnormal states also show clear regional differences, indicating that wind farms in different regions exhibit distinct operating-state distributions.  These proportions reflect the heuristic partition induced by the labeling rules and should not be interpreted as independently validated estimates of operational-state prevalence.

\begin{table}[!htbp]
\centering
\caption{The average proportion of rule-derived power-state pseudo-labels in each region.}
\label{tab:wind_power_States}
\begin{tabular}{lccc}
\toprule
\multicolumn{1}{c}{Region} & 
\multicolumn{1}{c}{Regular (\%)} & 
\multicolumn{1}{c}{Shutdown (\%)} & 
\multicolumn{1}{c}{Curtailment (\%)}\\
\midrule
AH & 84.3 & 14.7 & 1.0 \\
SD & 61.0 & 34.0 & 5.0 \\
SX & 78.2 & 20.8 & 1.0 \\
UK & 60.5 & 17.9 & 21.6 \\
\bottomrule
\end{tabular}
\end{table}

\subsubsection{Performance under Varied Conditions}
\label{sec:varied_conditions}
Figure~\ref{fig:condition_performance} compares UniWind with XGBoost, PatchTST, and CrossViViT under different wind-speed and pseudo-labeled state conditions.

We evaluate robustness under joint meteorological and pseudo-labeled conditions by dividing the SD\_A test samples into four conditions. Wind speeds below \(9\,\mathrm{m/s}\), corresponding to the 90th percentile on SD\_A, are treated as low-wind conditions, whereas the remaining samples are treated as high-wind conditions. We use the pseudo-labels defined in Section~\ref{sec:state_prior_labels} to distinguish regular and abnormal states, grouping shutdown-like and curtailment-like samples into the abnormal category. As shown in Figure~\ref{fig:condition_performance}, UniWind consistently achieves the lowest MAE across all four conditions, demonstrating its robustness to varying wind speeds and operating states. Under low-wind conditions, XGBoost is the strongest baseline in regular states, while the performance differences among baselines become relatively small in abnormal states. Under high-wind conditions, prediction errors increase substantially across all models, with PatchTST exhibiting particularly severe degradation, especially in abnormal states. In comparison, XGBoost and CrossViViT show more moderate error increases, with CrossViViT slightly outperforming XGBoost under high-wind abnormal conditions. Notably, UniWind maintains a clear advantage over all baselines in both high-wind scenarios, with the most pronounced improvements observed under high-wind abnormal conditions.
This result is consistent with the design motivation of UniWind: the meteorology-driven prior provides a structured reference for wind-dependent generation potential, while the correction module models deviations from that reference under different rule-derived power states.

\subsubsection{Physical-envelope Coverage Analysis}
To examine whether the bounded correction mechanism materially restricts the forecast, we analyze the timestamps at which the observed power exceeds the learned physical prior on SD\_A.   Such cases account for only 1.88\% of the test timestamps, indicating that the upper-envelope constraint is rarely activated as a restrictive ceiling.   Instead, it primarily provides a physically motivated reference that enables the state-aware correction module to adjust the prior according to different operating conditions.   This analysis supports interpreting the learned prior as a structured physical envelope for available generation, rather than as a direct estimate of realized power.

\begin{figure*}[!htb]
    \centering
    \includegraphics[width=0.8\linewidth]{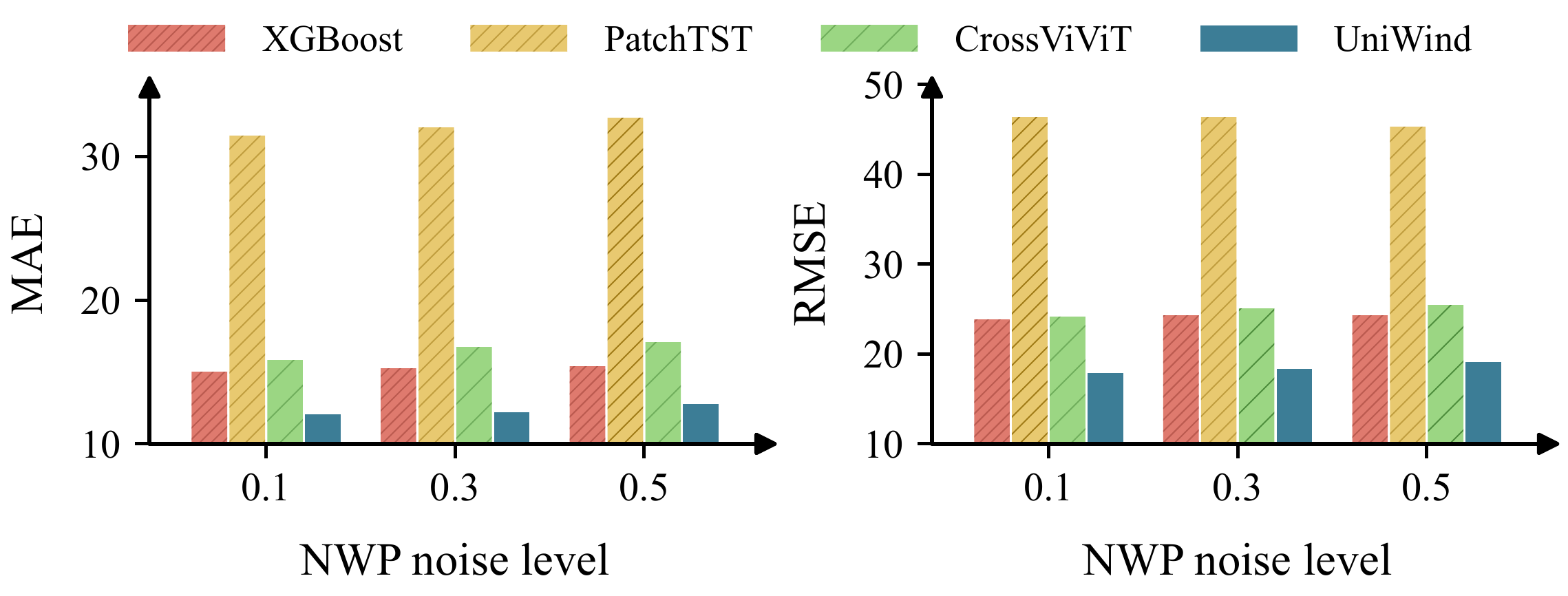}
    \caption{Robustness to NWP perturbations on the SD\_A dataset.}
    \label{fig:nwp_robustness}
\end{figure*}

\subsection{Case Studies}
Figure~\ref{fig:case_study} provides a qualitative example on the SD\_A dataset, illustrating how UniWind separates a physically motivated upper-envelope prior from realized generation. The physical prior generally follows the wind-speed profile: it increases rapidly when wind speed enters the productive range, decreases under low-wind conditions, and forms a plateau during high-wind periods. This plateau does not indicate a forecasting error, but reflects the saturation of the wind farm at its rated power after the wind speed reaches the rated region. However, the realized power can remain substantially below this available-power envelope because of latent operating conditions and site-specific effects. UniWind therefore does not directly use the physical prior as the final output. Instead, the State-aware Power Corrector adjusts the prior according to the inferred operating state, enabling the prediction to follow the ground-truth trajectory during both high-potential intervals and low-output periods. In particular, the prediction tracks the ground truth when the physical prior stays near the rated-power plateau, and also captures the sharp drops and recoveries when the available power changes rapidly. This case illustrates behavior consistent with the intended structured decomposition.

\begin{table}[!htbp]
\caption{Efficiency comparison on the SD\_A dataset.}
\label{table:Efficiency}
\small
\centering
\renewcommand{\arraystretch}{1.1}
\setlength{\tabcolsep}{0.5mm}
\resizebox{\linewidth}{!}{
\begin{tabular}{c c c c c}
\hline
 & WindFM & Chronos-2 & Moirai & UniWind\\
\hline
Parameters (M) & 8.1 & 119.5 & 13.8 & 4.0\\
Inference time (ms/sample) &  1319.07 & 25.06 & 26.35 & 0.72\\
\hline
\end{tabular}
}
\end{table}

\subsection{Efficiency and Robustness}
\subsubsection{Efficiency Comparison}
\label{sec:efficiency}

Table~\ref{table:Efficiency} compares the computational efficiency of UniWind with representative foundation-model baselines on the SD\_A dataset. UniWind has the smallest model size, with only 4.0M parameters, compared with 8.1M for WindFM, 119.5M for Chronos-2, and 13.8M for Moirai. In terms of inference efficiency, UniWind achieves the lowest latency of 0.72 ms per sample, making it approximately 35, 37, and 1,832 times faster than Chronos-2, Moirai, and WindFM, respectively. The substantially higher inference latency of WindFM is mainly attributed to the autoregressive forecasting strategy adopted by its evaluated checkpoint, which generates predictions sequentially and thus incurs additional computational overhead. Overall, these results demonstrate that UniWind achieves a favorable balance between model compactness and inference efficiency, making it well suited for repeated day-ahead forecasting across multiple wind farms.

\subsubsection{Robustness to NWP Perturbations}
\label{sec:nwp_robustness}
To assess the robustness of the models to uncertainty in NWP inputs, we conducted a controlled perturbation experiment on the SD\_A dataset. After standardization, we added zero-mean Gaussian noise with standard deviations of 0.1, 0.3, and 0.5 in standardized units. The same perturbed inputs were provided to XGBoost, PatchTST, CrossViViT and UniWind, and performance was evaluated using MAE and RMSE.

Figure~\ref{fig:nwp_robustness} shows the results. UniWind achieved the lowest MAE and RMSE at all three perturbation levels. As the noise standard deviation increased from 0.1 to 0.5, its MAE increased from 12.04 to 12.79, whereas its RMSE increased from 17.87 to 19.10. Despite the progressively stronger input perturbations, UniWind retained a consistent performance advantage over the baseline methods. These results suggest that the physical prior and state-aware correction help reduce the model's sensitivity to NWP uncertainty.

\section{Related Work}

\subsection{Physical and Data-driven Wind Power Forecasting}
Current research on wind power forecasting primarily relies on physical models or data-driven approaches.
Physical models~\cite{focken2001previento,li2013physical,wang2018short} estimate wind power by propagating meteorological forecasts through physical and engineering assumptions.
For example, Li et al.~\cite{li2013physical} first simulate steady-state flow fields under discrete inflow wind conditions to construct a database, and then predict short-term wind power by querying this database using wind inputs.
Data-driven models learn mappings from meteorological variables, historical power, and site features to future generation.
Statistical methods~\cite{sideratos2007advanced,ozkan2015novel,buhan2015multistage,phan2021hybrid,ju2019model} typically rely on handcrafted feature engineering. Tree-based approaches, such as XGBoost~\cite{phan2021hybrid} and LightGBM~\cite{ju2019model}, combine diverse meteorological features, and often serve as strong practical baselines. Deep learning models~\cite{fan2020m2gsnet,karijadi2023wind,keisler2024winddragon,zhang20252dxformer,chang2025multivariate} learn nonlinear weather-power dependencies from large-scale observations. More recently, following the emergence of foundation-model research~\cite{fang2026unraveling,wang2025lightgts}, WindFM~\cite{fan2025windfm} has been developed for wind energy prediction in unseen scenarios.
Although these methods capture rich empirical correlations, real-world wind-farm data often contain measurement noise, local environmental effects, and latent operational states, such as shutdowns and curtailment.
Consequently, physical models may fail to adapt to site-specific local conditions, whereas data-driven models may absorb hidden operational effects as noise or spurious correlations.

\subsection{Solar and Renewable-energy Forecasting}
Solar power forecasting models use weather conditions and historical observations, but rely more on cloud-sensitive multimodal signals~\cite{boussif2023improving,ma2024fusionsf,wang2025solarmae,li2024solarcube}. For example, CrossViViT~\cite{boussif2023improving} uses satellite imagery as spatio-temporal context and combines it with station time series for probabilistic forecasting. FusionSF~\cite{ma2024fusionsf} discretizes heterogeneous modalities into a shared codebook for robust solar power forecasting. However, wind power depends mainly on wind-speed dynamics rather than irradiance, and its uncertainty is more affected by turbine operation than cloud motion. These differences limit the direct transfer of solar forecasting models to wind power prediction.

\subsection{Long-horizon and Foundation-model Forecasting}
Time-series forecasting is widely used for wind power prediction because it directly learns temporal dependencies from historical sequences~\cite{zhao2016novel,xu2025cross,wangtimemixer}. Recent models such as PatchTST~\cite{nietime}, iTransformer~\cite{liuitransformer}, and xPatch~\cite{stitsyuk2025xpatch} improve long-horizon forecasting through patch-based, variable-centric, and multi-period representations. Foundation models such as Chronos~\cite{ansari2025chronos} and Moirai~\cite{liumoirai} leverage large-scale heterogeneous time-series pretraining for downstream adaptation. Nonetheless, these generic sequence forecasting models are often dominated by historical power patterns and may inadequately capture how future meteorological changes drive wind generation.

\subsection{Positioning of UniWind}
UniWind lies at the intersection of physics-based and data-driven wind power forecasting, while explicitly separating meteorology-driven generation potential from realized power. Unlike physical approaches that use a prescribed wind-to-power relationship as the final forecast, UniWind adapts a shared physical power curve to each site and treats its output as a soft upper envelope. Unlike end-to-end forecasting models that learn a direct mapping from weather and historical observations to power, UniWind uses this envelope as an intermediate representation and models deviations from it through a separate state-aware correction pathway.
This decomposition distinguishes UniWind from generic long-horizon and foundation-model forecasting approaches: future NWP conditions guide the retrieval of latent state representations from historical weather--state patterns, and a router combines bounded experts to account for state-dependent deviations. Rule-derived state labels provide weak supervision during training, but are not treated as verified operational annotations and are not required at inference time. UniWind therefore uses physical structure to represent available generation, learned temporal representations to capture latent operating effects, and bounded correction to produce realized-power forecasts. We evaluate this design in full-shot settings and in cross-farm and cross-region zero-shot settings.

\section{Conclusion}
\label{sec:conclusion}
In this paper, we propose UniWind, a day-ahead wind power forecasting model that links meteorological conditions to available power through a physical prior and then predicts realized generation through state-aware correction. UniWind first constructs a site-calibrated physical prior through a site-conditioned monotone warp, a shared physical power curve, and a physical upper-bound constraint. It then uses a State-aware Power Corrector to transform this available-power prior into final forecasts through knowledge-guided state routing and bounded state-specific correction. Extensive experiments on 24 real-world wind-farm datasets demonstrate that UniWind achieves strong full-shot and zero-shot performance and remains effective across diverse operating conditions. We believe that UniWind can promote the safe, economical, and efficient integration of wind power into future low-carbon energy systems.

Like other NWP-driven day-ahead forecasting methods, UniWind still relies on the quality of NWP inputs. However, the proposed site-calibrated physical prior and state-aware correction help mitigate this dependence.  Future work will explore the integration of additional meteorological predictors or NWP sources to better capture local weather dynamics and further improve forecasting reliability.

%\begin{acknowledgements}
%\end{acknowledgements}

\bibliographystyle{spmpsci}      % mathematics and physical sciences
\bibliography{ref}

\end{document}